\documentclass{article}
\usepackage{abstract}
\usepackage{amsfonts,amsmath,amssymb,amsthm}
\usepackage{booktabs}
\usepackage{caption}
\usepackage{dsfont}
\usepackage{eucal} 
\usepackage{graphicx}
\usepackage{hyperref,cleveref}
\usepackage{mathtools}
\usepackage{natbib}
\usepackage{soul}
\usepackage{stmaryrd}
\usepackage{thmtools,thm-restate}
\usepackage{tikz}
\usepackage{xcolor}

\usetikzlibrary{positioning}

\newif\iffig

\newlength{\zparindent}
\DeclareMathOperator*{\E}{\mathbb E}
\DeclareMathOperator*{\p}{\mathbb P}

\newcommand{\abs}[1]{{\left| #1 \right|}}
\newcommand{\norm}[1]{{\left\| #1 \right\|}}
\newcommand{\nnorm}[1]{{\left\vert\kern-0.25ex\left\vert\kern-0.25ex\left\vert #1 
    \right\vert\kern-0.25ex\right\vert\kern-0.25ex\right\vert}}

\newcommand{\braces}[1]{{\left\{ #1 \right\}}}
\newcommand{\brackets}[1]{{\left[ #1 \right]}}
\newcommand{\paren}[1]{{\left( #1 \right)}}

\declaretheorem{corollary}

\declaretheorem{lemma}

\declaretheorem{theorem}

\newtheorem{auxlemma}{Lemma}
\newenvironment{lemma*}[1]
  {\renewcommand\theauxlemma{#1}\auxlemma}
  {\endauxlemma}

\newcommand{\pd}[2]{{\frac{\partial #1}{\partial #2}}}
\newcommand{\dd}[2]{{\frac{\mathrm{d} #1}{\mathrm{d} #2}}}
\newcommand{\ind}{\mathds{1}}

\newcommand{\T}{^\intercal}

\newcommand{\R}{\mathbb R}

\newcommand{\cV}{\mathcal V}
\newcommand{\cW}{\mathcal W}
\newcommand{\cX}{\mathcal X}
\newcommand{\cY}{\mathcal Y}

\newcommand{\f}[1]{\mathbf{#1}}
\newcommand{\fa}{\mathbf a}

\newcommand{\fq}{\mathbf q}

\newcommand{\fu}{\mathbf u}
\newcommand{\fv}{\mathbf v}
\newcommand{\fw}{\mathbf w}
\newcommand{\fx}{\mathbf x}

\newcommand{\fz}{\mathbf z}

\newcommand{\fA}{\mathbf A}

\newcommand{\fD}{\mathbf D}

\newcommand{\fI}{\mathbf I}

\newcommand{\fP}{\mathbf P}
\newcommand{\fQ}{\mathbf Q}
\newcommand{\fR}{\mathbf R}

\newcommand{\fW}{\mathbf W}
\newcommand{\fX}{\mathbf X}

\usepackage{microtype}
\usepackage{subfigure}

\DeclareGraphicsExtensions{.pdf}

\usepackage{xspace}
\makeatletter
\DeclareRobustCommand\onedot{\futurelet\@let@token\@onedot}
\def\@onedot{\ifx\@let@token.\else.\null\fi\xspace}

\def\iid{{i.i.d}\onedot}
\def\eg{{e.g}\onedot} 
\def\ie{{i.e}\onedot}

\makeatother

\usepackage[accepted]{icml2019}

\begin{document}

\twocolumn[
\icmltitle{Towards Understanding Knowledge Distillation}



\icmlsetsymbol{equal}{*}

\begin{icmlauthorlist}
\icmlauthor{Mary Phuong}{ist}
\icmlauthor{Christoph H. Lampert}{ist}
\end{icmlauthorlist}

\icmlaffiliation{ist}{IST Austria (Institute of Science and Technology Austria)}

\icmlcorrespondingauthor{Mary Phuong}{bphuong@ist.ac.at}

\icmlkeywords{Machine Learning, Deep Learning, Distillation, Theory}

\vskip 0.3in
]



\printAffiliationsAndNotice{}  

\begin{abstract}
Knowledge distillation, i.e. one classifier being trained on the 
outputs of another classifier, is an empirically very successful 
technique for knowledge transfer between classifiers. It has even 
been observed that classifiers learn much faster and more reliably 
if trained with the outputs of another classifier as soft labels, 
instead of from ground truth data. So far, however, there is no 
satisfactory theoretical explanation of this phenomenon. In this work, 
we provide the first insights into the working mechanisms of 
distillation by studying the special case of linear and deep 
linear classifiers. Specifically, we prove a generalization 
bound that establishes fast convergence of the expected risk of a 
distillation-trained linear classifier. From the bound and its proof we 
extract three key factors that determine the success of distillation: 
\emph{data geometry} -- geometric properties of the data distribution, 
in particular class separation, has an immediate influence on the 
convergence speed of the risk;
\emph{optimization bias} -- gradient descent optimization finds a very 
favorable minimum of the distillation objective; 
and \emph{strong monotonicity} -- the expected risk of the student 
classifier always decreases when the size of the training set grows.
\end{abstract}

\section{Introduction}



In 2014, \citet{hinton14} made a surprising observation: 
they found it \emph{easier} to train classifier using the 
real-valued outputs of another classifier as target values 
than using actual ground-truth labels. 
Calling the procedure \emph{knowledge distillation}, or 
\emph{distillation} for short, they noticed the positive effect to occur 
even when the existing classifier (called \emph{teacher}) was 
trained on the same data as it used afterwards for the 
distillation-training of the new classifier (called \emph{students}). 
Since that time, the positive properties of distillation-based 
training has been confirmed several times: the optimization step 
is generally more well-behaved than the optimization step in 
label-based training, and it needs less if any regularization or 
specific optimization tricks. 
Consequently, in several fields, distillation has
become a standard technique for transfering the information 
between classifiers with different architectures, such as from 
deep to shallow neural networks or from ensembles of classifiers 
to individual ones. 

While the practical benefits of distillation are beyond 
doubt, its theoretical justification remains almost completely unclear.
Existing explanations rarely go beyond qualitative statements, 
\eg claiming that learning from soft labels should be easier 
than learning from hard labels, or that in a multi-class setting 
the teacher's output provides information about how similar different 
classes are to each other.
%

%

In this work, we follow a different approach. 
Instead of studying distillation in full generality, we restrict 
our attention to a simplified, analytically tractable, setting: 
binary classification with linear teacher and linear student 
(either shallow or deep linear networks). 
%
For this situation, we achieve the first 
quantitative results about the effectiveness of distillation-based 
training. 
Specifically, our main results are: \textbf{1) We prove a 
generalization bound that establishes extremely fast 
convergence of the risk of distillation-trained classifiers}. 
In fact, it can reach zero risk from finite training sets. 
\textbf{2) We identify three key factors that explain the 
  success of distillation}:
\emph{data geometry} -- geometric properties of the data distribution, 
in particular class separation, directly influence the convergence 
speed of the student's risk;
\emph{optimization bias} -- even 
though the distillation objective can have many optima, 
gradient descent optimization is guaranteed to find a 
particularly favorable one;
and \emph{strong monotonicity} -- increasing the training set 
always decreases the risk of the student classifier. 

\section{Related Work}

Ideas underpinning distillation have a long history dating back to the work of \citet{ba14,bucilua06,craven96,li14,liang08}.
In its current and most widely known form, it was introduced by \citet{hinton14} in the context of neural network compression.

Since then, distillation has quickly gained popularity among practitioners and established its place in deep learning folklore.
It has been found to work well across a wide range of applications, including e.g.
transferring from one architecture to another \cite{geras15},
compression \cite{howard17,polino18},
integration with first-order logic \cite{hu16} or other prior knowledge \cite{yu17},
learning from noisy labels \cite{li17},
defending against adversarial attacks \cite{papernot16},
training stabilization \cite{romero15,tang16},
distributed learning~\cite{polino18}, reinforcement
learning~\cite{rusu15} and data privacy~\cite{celik2017patient}.

In contrast to the empirical success, the mathematical principles 
underlying distillation's effectiveness have largely remained a 
mystery.
Only very works examine 
distillation from a theoretical perspective.
\citet{lopezpaz16} cast distillation as a form of learning using privileged information (LUPI, \citealt{vapnik15}),
a learning setting in which additional per-instance information is available at training time but not at test time.
However, the LUPI view concentrates on 
the aspect that the teacher's supervision to the student is noise-free. 
This argument fails to explain, \eg, the success 
of distillation even when the original problem is noise-free to start 
with.
The only other theoretical analysis we are aware of is by \citet{urner11}, who study distillation as a form of semi-supervised learning.
Specifically, they show that a two-step procedure, consisting of first training a teacher on a small labelled dataset and then training the student on a separate large dataset labelled by the teacher, can be more effective than training the student directly on the small labelled dataset.
The paper's focus is on the semi-supervised aspect, i.e.\ the gains from having a large unlabelled dataset.

%


A more distantly related topic is \emph{machine teaching}~\cite{zhu15}.
In machine teaching, a machine learning system is trained by a 
human teacher, whose goal is to hand-pick as small a training 
set as possible, while ensuring that the machine learns a desired 
hypothesis.
Transferring knowledge via machine teaching techniques is extremely effective:
perfect transfer is often possible from a small finite teaching set~\cite{zhu13,liu16}. 
However, the price for this radical reduction in sample complexity is the expensive 
training set construction.
Our work shows that, at least in the linear setting, distillation 
achieves a similar effectiveness with a more practical form of 
supervision.

\section{Background: Linear Distillation} \label{sec:setup}
We formally introduce distillation in the context of binary 
classification.
Let $\cX\subseteq\R^d$ be the input space, $\cY=\{0,1\}$ the label space, 
and $P_{\fx}$ the probability distribution of inputs.
We assume $P_\fx$ has a density.

The \emph{teacher} $h^*:\cX\to\cY$ is a fixed linear classifier, 
\ie $h^*(\fx) = \ind\braces{\fw_*\T \fx \geq 0}$ for some 
$\fw_*\in\R^d\setminus\braces{\f0}$, where $\ind\braces{.}$ 
returns 1 if the argument is true and 0 otherwise.
The \emph{student} also is a linear classifier, 
$h(\fx) = \ind\braces{\fw\T\fx \geq 0}$. 

We allow the weight vector to be parameterised as a product of 
matrices, $\fw\T=\fW_N\fW_{N-1}\cdots\fW_1$ for some $N\geq 1$.
When $N\geq 2$, this parameterisation is known as a 
\emph{deep linear network}.
Although deep linear networks have no additional capacity compared 
to directly parameterised linear classifiers $(N=1; \fw\T=\fW_1)$, they induce 
different gradient-descent dynamics, and are often studied as a 
first step towards understanding deep nonlinear networks~\cite{saxe14,kawaguchi16,hardt17}.

\emph{Distillation} proceeds as follows.
First, we collect a \emph{transfer set} $\braces{(\fx_i, y_i) }_{i=1}^n$ 
consisting of inputs $\fx_i$ sampled \iid from $P_\fx$, and 
\emph{soft labels} $y_i=\sigma(\fw_*\T\fx_i)$ provided by the 
teacher, where $\sigma$ is the sigmoid function, $\sigma(x) = 1/(1+\exp(-x))$.
The soft (real-valued) labels can be thought of as a more informative version 
of the hard (0/1-valued) labels of the standard classification setting. 
We write $\fX=\brackets{\fx_1,\dots,\fx_n}\in\R^{d\times n}$ for 
the data matrix. 
Second, the student is trained by minimizing the (normalized) 
cross-entropy loss,
\begin{multline} \label{eq:loss} 
  L^1(\fw) = -\frac 1 n \sum_{i=1}^n \Big[ y_i\log \sigma(\fw\T\fx_i)
  \\ + (1-y_i)\log(1-\sigma(\fw\T\fx_i))\Big] - L^*,
\end{multline}
where $L^*$ is a normalization constant, such that the minimum 
of $L^1$ is 0.
It only serves the purpose of simplifying notation and has no 
effect on the optimization.
%

The student observes the loss as a function of its parameters, 
\ie the individual weight matrices, 
\begin{equation}
  L(\fW_1,\dots,\fW_N) := L^1((\fW_N\fW_{N-1}\cdots\fW_1)\T),
\end{equation}
and optimizes it via gradient descent.
For the theoretical analysis, we avoid the complications of stepsize 
selection and adopt the notion of \emph{infinitesimal step 
size}\footnote{For readers who are unfamiliar with gradient flows, it 
suffices to think of the stepsize as finite and "sufficiently small".}, 
which turns the gradient descent procedure into a continuous \emph{gradient flow}.
We write $\fW_i(\tau)$ for the value of the matrix $\fW_i$ at time 
$\tau\in [0,\infty)$, with $\fW_i(0)$ denoting the initial value,  
and $\fw(\tau)\T = \fW_N(\tau)\cdots\fW_1(\tau)$.
Then, each $\fW_i(\tau)$, for $i\in\braces{1,\dots,N}$, evolves according 
to the following differential equation. \vspace*{-2mm}
\begin{equation}\label{eq:gd}
  \pd{\fW_i(\tau)}{\tau} = -\pd{L}{\fW_i}(\fW_1(\tau),\dots,\fW_N(\tau)).
\end{equation}
The student is trained until convergence, \ie $\tau\to\infty$.
We measure the \emph{transfer risk} of the trained student, defined as 
the probability that its prediction differs from that of the teacher,
\begin{equation}
  R(h) = \p_{\fx\sim P_\fx}\brackets{ h(\fx) \neq h^*(\fx)}.
\end{equation}
In Section~\ref{subsec:transfer-rates}, we will derive a bound for the transfer risk and establish how rapidly it decreases as a function of $n$.

\section{Generalization Properties of Linear Distillation}
This section contains our main technical results. 
First, in Section~\ref{subsec:closedform}, we provide an explicit 
characterization of the outcome of distillation-based training in 
the linear setting. In other words, we identify \emph{what the student 
actually learns}. In particular, we prove that the student is able 
to perfectly identify the teacher's weight vector, if the number 
of training examples ($n$) is equal to the dimensionality of 
the data ($d$) or higher. 
If less data is available, under minor assumptions, the student 
finds the best approximation of the teacher's weight vector that 
is possible within the subspace spanned by the training data. 

In Section~\ref{subsec:transfer-rates} we use these 
results to study the generalization properties of the student 
classifier, \ie we characerize \emph{how fast the student learns}.
Specifically, we prove a generalization bound with much more 
appealing properties than what is possible in the classic 
situation of learning from hard labels. %
As soon as enough training data is available ($n\geq d$), 
the student's risk is simply $0$. 
Otherwise, the risk can be bounded explicitly in a 
distribution-dependent way that, in particular, allows 
us to identify three key factors that explain the success 
of distillation, and to understand when distillation-based 
transfer is most effective.

\subsection{What Does the Student Learn?}\label{subsec:closedform}
In this section, we derive in closed form the asymptotic 
solution to the gradient flow~(\ref{eq:gd}) undergone by 
the student when trained by distillation. 
We state the results separately for directly parameterized linear 
classifiers $(N=1)$ and deep linear networks $(N\geq 2)$, as the
settings require slightly different ways of initializing parameters.
Namely, in the former case, initializing $\fw(0)=\f0$ is valid,
while in the latter case, this would lead to 
vanishing gradients, and we have to initialize with small 
(typically random) values. 

\begin{restatable}{theorem}{closedformshallow} \label{thm:closed-form-N1}
  Assume the student is a directly parameterised linear classifier $(N=1)$ 
  with weight vector initialised at zero, $\fw(0)=\f0$. 
  Then, the student's weight vector fulfills almost surely 
  \begin{equation}
  \fw(t) \to \hat\fw,
  \end{equation}
   for $t\to\infty$, with
  \begin{equation}
    \hat\fw = \left\{
      \begin{array}{cl}
        \fw_*, &\ n\geq d, \\
        \fX(\fX\T\fX)^{-1}\fX\T\fw_*, &\ n < d.
      \end{array}\right.
  \end{equation}  
\end{restatable}

Theorem~\ref{thm:closed-form-N1} shows a remarkable property of 
distillation-based training for linear systems: if sufficiently 
many (at least $d$) data points are available, the student 
exactly recovers the teacher's weight vector, $\fw_*$.
This is a strong justification for distillation as a method of 
\emph{knowledge transfer} between linear classifiers and the 
theorem establishes that the effect occurs not just in the 
infinite data limit ($n\to\infty$), as one might have expected, 
but already in the finite sample regime ($n\geq d$). 

When few data points are available ($n<d$), the weight vector
learned by the student is simply the \emph{projection of the 
teacher's weight vector onto the data span} (the subspace spanned
by the columns of $\fX$).
In a sense, this is the best the student can do: the gradient descent
update direction $\pd{\fw(\tau)}{\tau}$ always lies in the data span, 
so there is no way for the student to learn anything outside of it.
The projection is the best subspace-constrained approximation of 
$\fw_*$ with respect to the Euclidean norm.
The extent to which Euclidean closeness implies closeness in 
predictions is a separate matter, and the subject of Section~\ref{subsec:transfer-rates}.

\begin{proof}[Proof sketch of Theorem~\ref{thm:closed-form-N1}]
  First, notice that $\hat\fw$ is a global minimiser of $L^1$.
  Moreover, when $n\geq d$, it is (almost surely wrt. $\fX\sim P_\fx^n$) unique, and when $n < d$, it is (almost surely) the only one lying in the span of $\fX$ and thus potentially reachable by gradient descent.

  The proof consists of two parts.
  We prove that
  a) the gradient flow~(\ref{eq:gd}) drives the objective value towards the optimum, 
$L^1(\fw(t)) \to 0$ as $t\to\infty$, and b)
  the distance between $\fw(t)$ and the claimed asymptote $\hat\fw$ is upper-bounded by the objective gap, \vspace*{-2mm}
  \begin{equation} \label{eq:dist-from-opt}
    \norm{\fw(t)-\hat\fw}^2 \leq cL^1(\fw(t))
  \end{equation}
  for some constant $c>0$ and all $t\in[0,\infty)$.

  For part a), observe that $L^1$ is convex.
  For any $\tau\in[0,\infty)$, the time-derivative of $L^1(\fw(\tau))$ is negative unless we are at a global minimum,
  \begin{align}
    \begin{split} \label{eq:dLdt}
    \dd{}{\tau} L^1(\fw(\tau))
    &= \nabla L^1(\fw(\tau))\T \paren{\pd{\fw(\tau)}{\tau}} \\
    &= - \norm{\nabla L^1(\fw(\tau))}^2,
    \end{split}
  \end{align}
  implying that the objective value $L^1(\fw(\tau))$ decreases monotonically in $\tau$.
  Hence, if we denote by $\cW = \braces{\fw: L^1(\fw)\leq L^1(\f0)}$ the $L^1(\f0)$-sublevel set of the objective, we know that $\fw(\tau)\in\cW$ for all $\tau\in[0,\infty)$.
  One can show that on this set, $L^1$ satisfies strong convexity, but only along certain directions: for some $\mu>0$ and all $\fw, \fv\in\cW$ such that $\fv-\fw\in\mathrm{span}(\fX)$,
  \begin{multline} \label{eq:strong-convexity}
    L^1(\fv) \geq L^1(\fw) + \nabla L^1(\fw)\T(\fv-\fw) + \frac{\mu} 2 \norm{\fv-\fw}^2.\!\!\!\!
  \end{multline}
  This allows us (via a technical derivation that we omit here) to relate the objective gap to the gradient norm:
  it can be shown that there exists $c'>0$, such that
  \begin{equation}
    c'L^1(\fw) \leq \frac 1 2 \norm{\nabla L^1(\fw)}^2 .
  \end{equation}
  Applying the above to $\fw(\tau)$ in (\ref{eq:dLdt}), we are able to bound the amount of reduction in the objective in terms of the objective itself, ultimately proving linear convergence.

  For part b),
  invoke~(\ref{eq:strong-convexity}) with $\fv=\fw(\tau)$ and $\fw=\hat \fw$; this gives
$
    L^1(\fw(\tau)) \geq \frac \mu 2 \norm{\fw(\tau)-\hat\fw}^2.
 $
\end{proof} 
The full proof is given in the Supplementary Material.

The next results is the analog of Theorem~\ref{thm:closed-form-N1} 
for deep linear networks. Here, some technical conditions 
are needed because the parameters cannot all be initialized at $0$. 

\begin{restatable}{theorem}{closedformdeep} \label{thm:closed-form}
  Let $\hat\fw$ be defined as in Theorem~\ref{thm:closed-form-N1}.
  Assume the student is a deep linear network, initialized such that
  for some $\epsilon>0$, 
  \begin{align}
    \norm{\fw(0)} < \min\Big\{ \norm{\hat\fw}, \epsilon^{N} & \paren{\epsilon^2 \norm{\hat\fw}^{-\frac 2N} + \norm{\hat\fw}^{2-\frac {2}N}}^{-\frac N2} \Big\},
                    \label{ass:near-zero} \\
    L^1(\fw(0)) &< L^1(\f0), \label{ass:better-than-zero} \\
    \fW_{j+1}(0)\T\fW_{j+1}(0) &= \fW_j(0)\fW_j(0)\T \label{ass:balancedness}
  \end{align}
  for $j=1,\dots,N-1$. 
  Then, for $n\geq d$, 
  student's weight vector fulfills almost surely 
  \begin{equation}
  \fw(t) \to \hat\fw,
  \end{equation}
  and for $n<d$, 
  \begin{equation}
  \norm{\fw(t) -  \hat\fw} \leq \epsilon,
  \end{equation}
  for all $t$ large enough.
\end{restatable}
The interpretation of the theorem is analogous to 
Theorem~\ref{thm:closed-form-N1}. 
Given enough data ($n\geq d$), the student learns to perfectly 
mimic the teacher. Otherwise, it learns an approximation at 
least $\epsilon$-close to the projection of the teacher's
weight vector onto the data span.

The conditions \eqref{ass:near-zero}--\eqref{ass:balancedness}
appear for technical reasons and a closer look at them shows
that they do not pose problems in practice.
Condition~(\ref{ass:near-zero}) states that the network's weights 
should be initialised with sufficiently small values. 
Consequently, this assumption is easy to satisfy in practice. 
%
%
Condition~(\ref{ass:better-than-zero}) requires that the 
initial loss is smaller than the loss at $\fw=\f0$.
This condition guarantees that the gradient flow does not hit 
the point $\fw=\f0$, where all gradient vanish and the optimization 
would stop prematurely. 
In practice, when the step size is finite, the condition is 
not needed. Nevertheless, it is also not hard to satisfy: for any 
near-zero initialisation, $\fw(0) = \fw_0$, either $\fw_0$ or $-\fw_0$ 
will satisfy (\ref{ass:better-than-zero}), so at most one has to flip 
the sign on one of the $\fW_i(0)$ matrices.
Finally, condition~(\ref{ass:balancedness}) is called \emph{balancedness}~\cite{arora18}
and discussed in-depth in~\cite{arora19}). It simplifies the analysis 
of matrix products and makes it possible to explicitly analyze the 
evolution of $\fw$ induced by gradient flow in the $\fW_i$'s.
Assuming near-zero initialization, the condition is automatically 
satisfied approximately and there is some evidence~\cite{arora19} 
suggesting that \emph{approximate balancedness} may suffice for 
convergence results of the kind we are interested in.
Otherwise, the condition can also simply be enforced numerically.

\begin{proof}[Proof sketch of Theorem~\ref{thm:closed-form}]
  First, we establish convergence in the objective, $L^1(\fw(t))\to 0$ as $t\to\infty$, similarly to the case $N=1$.
  Unlike that case, however, the evolution of the end-to-end weight vector $\fw(\tau)$ is governed by complex mechanics induced by gradient flow in $\fW_i$'s.
  A key tool for analyzing this induced flow was recently established in~\cite{arora18}:
  the authors show that the induced flow behaves similarly to gradient flow with momentum applied directly to $\fw$.
  Making use of this result, one can proceed analogously as in the case of $N=1$ to show convergence in the objective.
  
  Second, to show convergence in parameter space, 
  we decompose $\fw(t)$ into its projection onto the span of $\fX$, and an orthogonal component.
  The $\fX$-component converges to $\hat\fw$, by strong convexity arguments as in the case $N=1$.
  It remains to show that the orthogonal component is small.
  Now, recall that in the case $N=1$, we initialise at $\fw(0)=\f0$ and move within the span, so the orthogonal component is always zero.
  When $N\geq 2$, the situation is different:
  a) we initialise with a potentially non-zero orthogonal component (because we need to avoid the spurious stationary point $\fw=\f0$), and
  b) the momentum term causes the orthogonal component to grow during optimisation.
  Luckily, the rate of growth can be precisely characterised and controlled by the initialisation norm $\norm{\fw(0)}$,
  so depending on how close to zero we initialise, we can upper-bound the size of the orthogonal component.
  This yields a bound on the distance $\norm{\fw(t)-\hat\fw}$.
\end{proof}
For the formal proof, we refer the reader to the Supplemental Material.

\subsection{How Fast Does the Student Learn?}\label{subsec:transfer-rates}

In this section, we present our main quantitative result, 
a bound for the expected transfer risk in linear distillation.

%
%
%
%
%

We first introduce some geometric concepts. 
For any $\fu,\fv \in \R^d\setminus\{\f0\}$, denote 
by $\bar\alpha(\fu,\fv)\in[0,\pi/2]$ the unsigned angle 
between the vectors $\fu$ and $\fv$
\begin{equation}
  \bar\alpha(\fu, \fv) = \cos^{-1}\!\paren{\frac {\abs{\fu\T \fv}} {\norm{\fu} \cdot \norm{\fv}}}.
  \label{eq:baralpha}
\end{equation}
A key quantity for us is the angle between $\fw_*$ and a randomly chosen $\fx$, for $\fx\sim P_\fx$.
For a given transfer task $(P_\fx, \fw_*)$, we denote by $p$ the reverse cdf of $\bar\alpha(\fw_*,\fx)$, 
\begin{equation}
  p(\theta) = \p_{\fx\sim P_{\fx}}[\bar\alpha(\fw_*, \fx)\geq \theta]
  \qquad \text{for}\quad \theta\in[0,\pi/2].   \label{eq:ptheta}
\end{equation}

\begin{figure}[t]
  \hspace{4mm}\begin{tikzpicture}[scale=0.9]
    \node (cc) at (0,0)    {};
    \node[left= 2.2 of cc] (lc) {};
    \node[right= 2.8 of cc] (rc) {};
    \node[above=0.5 of cc] (uc) {};
    \node[below=0.5 of cc] (dc) {};

    \node[above= 1.5 of cc, anchor=west] (ca) {$\fw_*$};
    \node[above= 1.5 of lc, anchor=west] (la) {$\fw_*$};
    \node[above= 1.5 of rc, anchor=west] (ra) {$\fw_*$};

    \node[left =0.7 of lc] (ll) {};
    \node[right=0.7 of lc] (lr) {};
    \node[left =0.9 of cc] (cl) {};
    \node[right=0.9 of cc] (cr) {};
    \node[left =1.2 of rc] (rl) {};
    \node[right=1.2 of rc] (rr) {};

    \fill[gray!50] (lc) ellipse (0.7 and 1.4);
    \fill[gray!50] (rc) ellipse (1.4 and 0.7);
    \fill[gray!50] (uc) circle (0.8 and 0.5);
    \fill[gray!50] (dc) circle (0.8 and 0.5);

    \draw[->] (cc.center) -- (ca.west);
    \draw[->] (lc.center) -- (la.west);
    \draw[->] (rc.center) -- (ra.west);

    \draw (ll) -- (lr);
    \draw (cl) -- (cr);
    \draw (rl) -- (rr);

    \node[below=0.8 of lc, anchor=west] {$P_\fx$};
    \node[below=0.8 of cc, anchor=west] {$P_\fx$};
    \node[below=0.4 of rc, anchor=west] {$P_\fx$};

    \node[below=1.5 of lc] {Task A};
    \node[below=1.5 of cc] {Task B};
    \node[below=1.5 of rc] {Task C};
  \end{tikzpicture}
  \begin{tabular}{l} \\[-3mm] \hspace{-3mm}
    \includegraphics[scale=0.27]{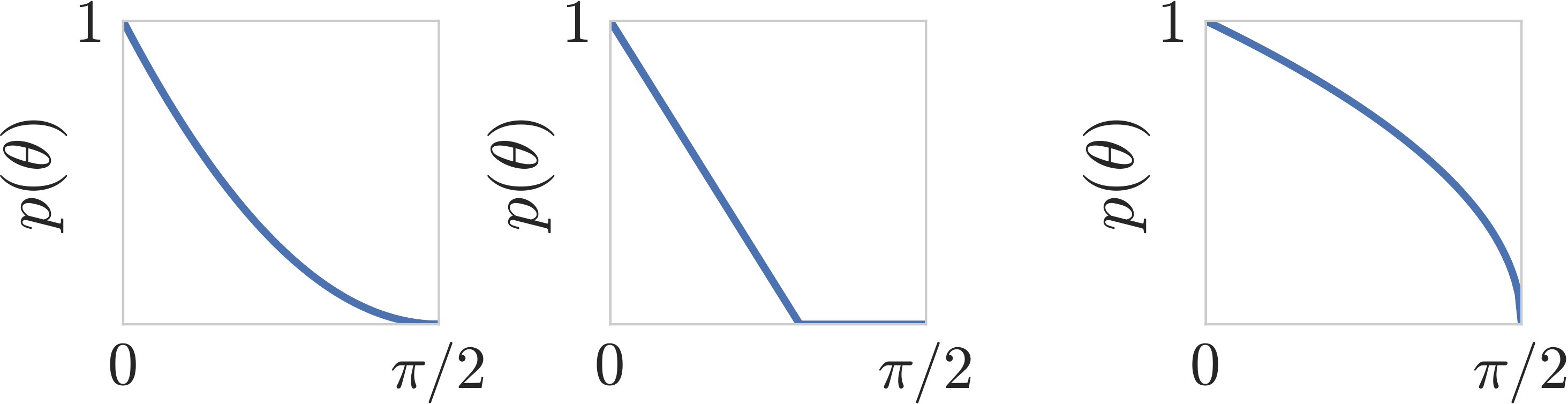}
  \end{tabular}
  \caption{Schematic illustration of $p(\theta)$ for three different transfer tasks.
  In Task A, the angular alignment between the data and the teacher's weight vector
  is high, so $p(\theta)$ is fast descreasing. In Task B, it is also high, and in 
  additional the classes are separated by a margin, so $p(\theta)$ reaches $0$ 
  before $\beta=\pi/2$. In Task C, the angular alignment is low, so $p(\theta)$ 
  decreases rather slowly.
  }
  \label{fig:angular-alignment}
\end{figure}
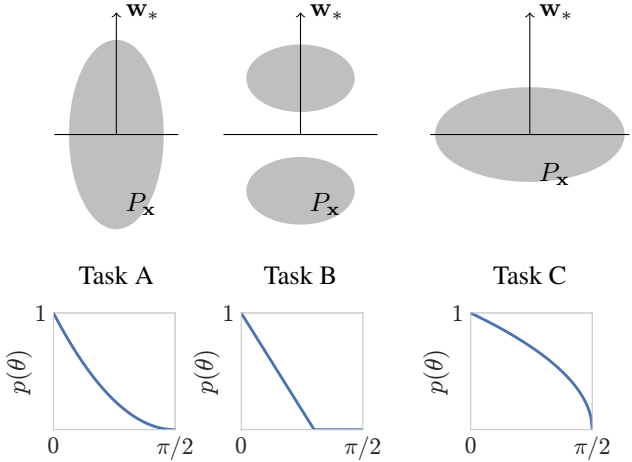
By construction, $p(\theta)$ is monotonically decreasing, starting with
$p(0)=1$ and approaches $0$ for $\theta\to\pi/2$. 
Figure~\ref{fig:angular-alignment} illustrates this behavior 
for three exemplary data distributions as \emph{Tasks A,B and C}.
In \emph{Task A}, the probability mass is well aligned with the 
direction of the teacher's weight vector. 
The probability that a randomly chosen data point $x\sim {P_\fx}$ 
has a large angle with $\fw_*$ is small. Therefore, the value 
of $p(\theta)$ quickly drops with growing angle $\theta$. 
In \emph{Task B}, the data also aligns well with $\fw_*$, but 
in addition, the data region remains bounded away from 
the decision boundary. Therefore, certain large angles 
can never occur, \ie there exists a value $\theta_0<\pi/2$, 
such that $p(\theta)=0$ for $\theta\geq \theta_0$.
In \emph{Task C}, the situation is different: the data 
distribution is concentrated along the decision boundary
and the probability of a angle between $\fw_*$ and a 
randomly chosen data point $x\sim {P_\fx}$ is large. 
As a consequence, $p(\theta)$ drops more slowly with growing 
angle than in the previous two settings.%

We are now ready to state the main result. 
For improved readability, we phrase it
for a student with infinitesimally small 
initialization, \ie $\epsilon\to 0$. 
The general formulation 
can be found in the supplemental material.

\begin{theorem}[Transfer risk bound for linear distillation] \label{thm:bound}
For any training set $\fX\in\R^{d\times n}$, let $\hat h_\fX(\fx) = \ind\braces{\hat\fw\T\fx \geq 0}$ 
be the linear classifier learned by distillation from a teacher with weight vector $\fw_*$.
Then, when $n\geq d$, it holds that 
\begin{align}
\E_{\fX\sim P_\fx^{\otimes n}}\brackets{ R\big(\hat h_\fX \big)} &=  0. \label{eq:boundzero} 
\intertext{For $n<d$, it holds for any $\beta\in[0, \pi/2]$ that}
\E_{\fX\sim P_\fx^{\otimes n}}\brackets{R\big(\hat h_\fX \big)} 
  &\leq    p(\beta) + p(\pi/2-\beta)^n \label{eq:bound}
  \end{align}
\end{theorem}
Equation~\eqref{eq:boundzero} is unsurprising, of 
course, because in Section~\ref{subsec:closedform} we 
already established that for $n\geq d$ the student is 
able to perfectly mimic the teacher. 


Inequality \eqref{eq:bound}, however, is --to our knowledge-- 
the first quantitative characterization how well a student 
can learn via distillation. 

Before we provide the proof sketch, 
we present two instantiations of the bound 
for specific classes of tasks that provide 
insight how fast the right hand side of 
\eqref{eq:bound} actually decreases. 

\begin{figure}[t]
  \centering
  \includegraphics[scale=0.32]{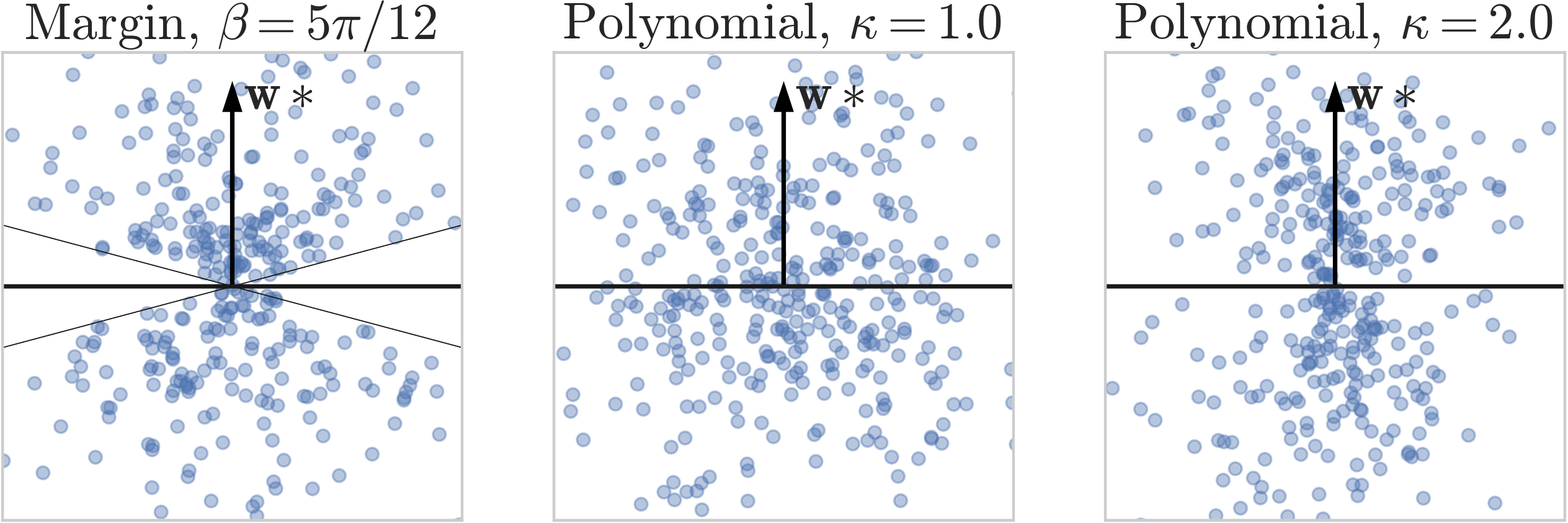}
  \caption{Examples of 2D distributions that fulfill the large-margin condition (left) and the polynomial condition with different parameters (center, right).}
  \label{fig:margin-poly}
\end{figure}

\noindent\textbf{The margin case.}
The first class of tasks we consider are tasks in which the 
classes are separated by an \emph{angular margin}, illustrated 
in Figure~\ref{fig:margin-poly} (left).
These tasks are characterized by a `wedge' of zero probability 
mass near the boundary\footnote{In bounded domains this 
condition is, in particular, fulfilled in the \emph{classical margin} 
situation~\cite{schoelkopf2002lwk}, when the classes are 
separated by a positive distance from each other.}. 
For these tasks, we obtain from Theorem~\ref{thm:bound} 
that the expected risk decays exponentially in $n$, up to $n=d-1$.

\begin{corollary}[Transfer risk of large-margin distributions] \label{cor:margin}
%
If there exists $\beta\in[0,\pi/2]$ such that $p(\beta)=0$ and $\gamma := p(\pi/2-\beta)<1$, then
  \begin{equation}
    \E_{\fX\sim P_\fx^n}\brackets{ R\big(\hat h_\fX \big)}
  \leq         \gamma^n.
  \end{equation}  
\end{corollary}

\noindent\textbf{The polynomial case.}
The second class are tasks for which we can upper-bound $p$ by a $\kappa$-order polynomial.
This can be done trivially for any task by setting $\kappa=0.0$, but that choice would yield a vacuous bound.
Higher values of $\kappa$ correspond to stronger assumptions on the distribution but enable better rates.
Figure~\ref{fig:margin-poly} (center, right) shows examples of polynomial distributions for $\kappa\in\braces{1.0, 2.0}$.
The special case $\kappa=1.0$ corresponds to a uniform angle distribution, while distribution with 
$\kappa=2.0$ have low probability mass near the decision boundary, while not necessarily exhibiting 
a margin. 

The following corollary establishes that for tasks with polynomial behavior of $p(\theta)$, 
the expected risk decays essentially at a rate of $(\log n/n)^\kappa$ or faster.

\begin{corollary}[Transfer risk of polynomial distributions] \label{cor:poly}
  If there exists a $\kappa\geq 0$ be such that $p(\theta) \leq c\cdot(1-(2/\pi)\theta)^\kappa$ 
  for all $\theta\in[0,\pi/2]$, then 
  \begin{equation}
    \E_{\fX\sim P_\fx^n}\brackets{ R\big(\hat h_\fX)}
       \leq c\cdot\frac{1 + (\log n)^\kappa}{n^\kappa} 
  \end{equation}  
\end{corollary}
\begin{proof}[Proof]
  We apply Theorem~\ref{thm:bound} and insert the polynomial 
  upper bound for $p$. For the case $n < d$, we get
  \begin{align}
    &\E_{\fX\sim P_\fx^n}\brackets{ R\big(\hat h_\fX \big)}   \nonumber
    \\
    &\leq (1-(2/\pi)\beta)^\kappa + (1-(2/\pi)(\pi/2 - \beta))^{n\kappa}.
\intertext{Setting $\beta = (\pi/2)\cdot n^{-1/n}$ and simplifying the resulting expressions
yields}
    &\leq
    \big(1-e^{-\frac{\log n}{n}}\big)^\kappa + n^{-\kappa}.
  \end{align}
  Finally, we use the inequality $e^x \geq 1+x$ and the claim follows.
\end{proof}


Note that, in contrast to many results in statistical learning theory,
the bounds are far from vacuous, even when only little data is 
available. 
This can best be seen in Corollary~\ref{cor:margin}, where $\gamma<1$ 
and hence $\gamma^n$ is an informative upper bound for the 
classification error.
These observations suggest that distillation operates in a very 
different regime from classical hard-target learning.
Standard bounds usually have little to say when $n<d$ and only 
start to be useful when $n\gg d$.
In contrast, (linear) distillation ensures perfect transfer when 
$n\geq d$ and non-vacuous bounds are possible even when $n<d$.

\subsection{Proof of Theorem 3}

The case $n\geq d$ follows trivially from the result of Theorem 1 and 2. 
%
For the case $n < d$,
the following property turns out to be crucial for obtaining a transfer rate of the form that we do.
\begin{restatable}[Strong monotonicity]{lemma}{improvement} \label{lemma:improvement}
  Let $\hat\fw(\fX)$ denote the distillation solution $\hat\fw$ as a function of the training data $\fX$.
  Then, for any full-rank datasets $\fX_-\in\R^{d\times n_-}$ and $\fX_+\in\R^{d\times n_+}$ such that $\fX_-$ is contained in $\fX_+$,
  \begin{equation}
   \bar\alpha(\fw_*, \hat\fw(\fX_+)) \leq \bar\alpha(\fw_*, \hat\fw(\fX_-)).  \label{eq:monotonicity}
  \end{equation}
\end{restatable}
\begin{proof}
  If $n_+\geq d$, then the left-hand side of (\ref{eq:monotonicity}) is zero and the claim follows.
  Otherwise,
  assume wlog that the first $n_-$ columns of $\fX_-$ and $\fX_+$ coincide.
  Let $\fQ_+\fR_+ = \fX_+$ be the QR factorisation of $\fX_+$ with $\fQ_+\in\R^{d\times n_+}$ and $\fR_+\in\R^{n_+\times n_+}$, and similarly for $\fX_-$.
  Then $\hat\fw(\fX_+) = \fQ_+\fQ_+\T\fw_*$ and
  \begin{align}
    \cos(\bar\alpha(\fw_*, \hat\fw(\fX_+)))
    & =  \frac{\fw_*\T \fQ_+\fQ_+\T\fw_*} {\norm{\fw_*}\cdot\norm{\fQ_+\fQ_+\T\fw_*}} \\
    & =  \frac{\norm{\fQ_+\T\fw_*}} {\norm{\fw_*}},
  \end{align}
  and an analogous statement holds for $\fX_-$.
  Now, because the first $n_-$ columns of $\fQ_+$ coincide with $\fQ_-$, we have $\norm{\fQ_+\T\fw_*} \geq \norm{\fQ_-\T\fw_*}$ and
  \begin{equation}
    \cos(\bar\alpha(\fw_*, \hat\fw(\fX_+))) \geq \cos(\bar\alpha(\fw_*, \hat\fw(\fX_-))).
  \end{equation}
  Taking $\cos^{-1}$ on both sides (and remembering that $\cos^{-1}$ is decreasing) yields the claim.
\end{proof}

For the moment, think of $\bar\alpha(\fw_*,\hat\fw)$ as a proxy for the 
transfer risk, \ie the closer the trained student $\hat\fw$ is to the 
teacher $\fw_*$ in terms of angles, the lower the transfer risk.
A direct consequence of Lemma~\ref{lemma:improvement}, and the reason 
we call it `strong mononoticity', is that including additional data 
in the transfer set can never harm the transfer risk, only improve it.
This property is specific to distillation;
it does not hold in hard-target learning.

\begin{proof}[Proof of Theorem~\ref{thm:bound} $(n<d)$]
  For nonzero vectors $\fu,\fv \in \R^d$, we define $\alpha(\fu,\fv)\in[0,\pi]$ 
  as a variant of $\bar\alpha$ (Equation~\ref{eq:baralpha}) that takes
  the sign of $\fu\T \fv$ into account,
\begin{equation}
  \alpha(\fu, \fv) = \cos^{-1} \paren{\frac {{\fu\T \fv}} {\norm{\fu} \cdot \norm{\fv}}}.
\end{equation}
%
  We decompose the expected risk as follows:
  \begin{align}\begin{split} \label{eq:risk-decomposition}
      \E_{\fX\sim P_\fx^n}&\brackets{R\big(\hat h_\fX \big)} = 
      \p_{\substack{\fX\sim P_{\fx}^n \\ \fx\sim P_{\fx} }}[\fw_*\T \fx \cdot \hat\fw\T \fx < 0] \\
      &= \int_{\fx: \bar\alpha(\fw_*, \fx) \geq \beta} \p_{\fX\sim P_{\fx}^n}[\fw_*\T \fx\cdot \hat\fw\T \fx <0 |\fx]\, \mathrm{d}P_{\fx} \\
      &+ \int_{\fx: \bar\alpha(\fw_*,\fx) < \beta ,\, \fw_*\T \fx > 0} \p_{\fX\sim P_{\fx}^n}[\hat\fw\T \fx <0 |\fx]\, \mathrm{d}P_{\fx} \\      
      &+ \int_{\fx: \bar\alpha(\fw_*,\fx) < \beta ,\, \fw_*\T \fx < 0} \p_{\fX\sim P_{\fx}^n}[\hat\fw\T \fx >0 |\fx]\, \mathrm{d}P_{\fx} .
    \end{split}\end{align}
  Let us fix some $\fx$ for which $\bar\alpha(\fw_*, \fx)<\beta$ and $\fw_*\T \fx > 0$ (\ie an `easy' positive test example); for this $\fx$ we have $\alpha(\fw_*,\fx)=\bar\alpha(\fw_*,\fx)$. Consider the situation where $\bar\alpha(\fw_*, \fx_i)< \pi/2-\beta$ for some $i$ (\ie there is at least one good teaching point).
  Then, Lemma \ref{lemma:improvement} with $\fX_+=\fX$ and $\fX_-=\fx_i$ yields $\bar\alpha(\fw_*,\hat\fw) \leq \bar\alpha(\fw_*,\fx_i) < \pi/2 - \beta$.
  Combined with the triangle inequality, we obtain
  \begin{align}
    \alpha(\hat\fw,\fx) &\leq \alpha(\fw_*, \hat\fw) + \alpha(\fw_*,\fx) \\
    &\leq \bar\alpha(\fw_*,\fx_i) + \bar\alpha(\fw_*,\fx) < \pi/2,
  \end{align}
  which implies $\hat\fw\T \fx > 0$, \ie a correct prediction (same as the teacher's).
  Conversely, an error can occur only if $\bar\alpha(\fw_*,\fx_i) \geq \pi/2 - \beta $ for all $i$.
  Because $\fx_i$ are independent, we have
  \begin{align}
    \begin{split} \label{eq:bound-trainset1}
    \p_{\fX\sim P_{\fx}^n}[\hat\fw\T \fx < &\, 0 |\,\fx: \bar\alpha(\fw_*,\fx)<\beta,\,  \fw_*\T \fx>0] \\
    &\leq \p_{\fX\sim P_{\fx}^n}[\forall_i: \bar\alpha(\fw_*,\fx_i) \geq \pi/2 - \beta] \\
    &= p(\pi/2-\beta)^n.
    \end{split}
  \end{align}
  By a symmetric argument, one can show that
  \begin{multline} \label{eq:bound-trainset2}
    \p_{\fX\sim P_{\fx}^n}[\hat\fw\T \fx >0 |\,\fx: \bar\alpha(\fw_*,\fx)<\beta,\,  \fw_*\T \fx<0] \\ \leq p(\pi/2-\beta)^n.
  \end{multline}
  Combining (\ref{eq:risk-decomposition}), (\ref{eq:bound-trainset1}) and (\ref{eq:bound-trainset2}) yields the result:
  \begin{align*}
    \p_{\substack{\fX\sim P_{\fx}^n \\ \fx\sim P_{\fx} }}&[\fw_*\T \fx \cdot \hat\fw\T \fx < 0] \leq \\
 \leq&\, \p_{\fx}[\bar\alpha(\fw_*,\fx)\geq\beta] + \p_{\fx}[\bar\alpha(\fw_*,\fx)<\beta] \cdot p(\pi/2-\beta)^n \\
 =&\, p(\beta) + (1-p(\beta))\cdot p(\pi/2-\beta)^n.
  \end{align*}
\end{proof}

\section{Why Does Distillation Work?}\label{sec:discussion}
From the formal analysis in the previous section, three 
concepts emerge as key factors for the success
of distillation: 
\emph{data geometry, optimization bias, and strong monotonicity}.
In this section, we discuss these factors and provide some 
empirical confirmation how they affect or explain variations 
in the transfer risk.


\subsection{Data Geometry} \label{sec:exp-data-geom}
From Theorem~\ref{thm:bound} we know that the data geometry, in 
particular the \emph{angular alignment} between the data distribution 
and the teacher, crucially impact how fast the student can learn. 
Formally, this is reflected in $p(\theta)$: the faster it 
decreases, the easier it should be for the student to learn 
the task.

To experimentally test the effect of data geometry on the 
effectiveness of distillation, we adopt the setting of 
Corollary~\ref{cor:poly}. We consider a series of tasks 
of varying angular alignment, as measured by the degree, 
$\kappa$, of the polynomial by which $p(\theta)$ is upper bounded.  



Specifically, for any $\kappa$, the task $(P_\fx^\kappa, \fw_*^\kappa)$ 
is defined by the following sampling procedure.
First, an angle $a$ is sampled from the $\kappa$-polynomial 
distribution, i.e. $\p\brackets{a \geq \theta} = (1-(2/\pi)\theta)^\kappa$ for $\theta\in[0,\pi/2]$.
Then, a direction $\fz$ is uniformly sampled from all 
unit-length vectors that are at angle $a$ with the teacher's 
weight vector, $\bar\alpha(\fw_*,\fz) = a$.
Finally, $\fx = \nu \fz$ is returned for a random $\nu$, 
distributed as a one-dimensional standard Gaussian. 

We use an input space dimension of $d=1000$ and a transfer set size $n=20$.
Then, we train a linear student by distillation on each of the tasks and evaluate its transfer risk on held-out data.
Figure~\ref{fig:exp-data-geometry} shows the results.
The plot shows a clearly decreasing trend:
on tasks with more favorable data geometry (higher $\kappa$), transfer via distillation 
is more effective and the student achieves lower risk.

\begin{figure}[t]
  \centering
  \includegraphics[scale=0.35]{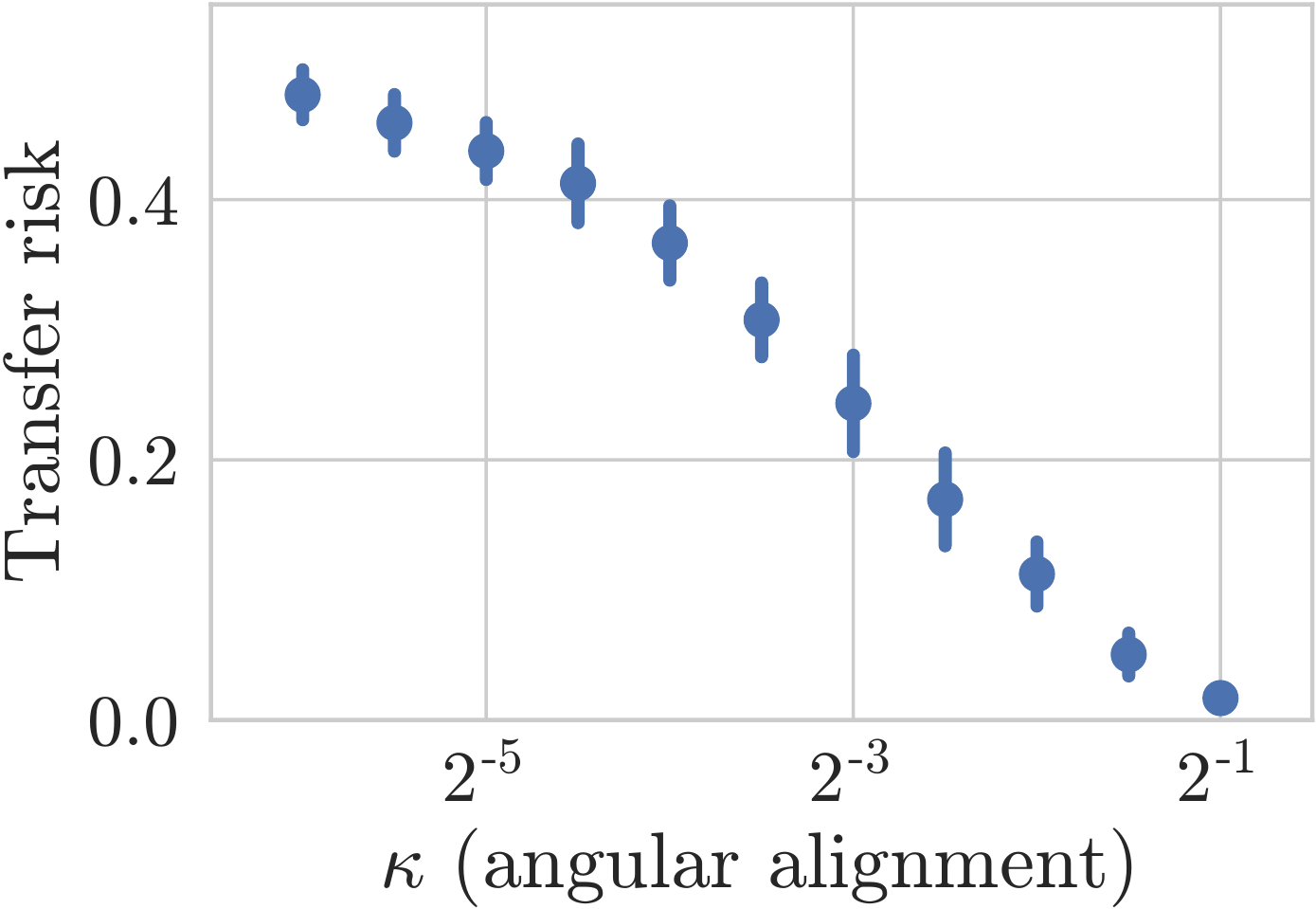}
  \caption{
    Transfer risk of linear distillation on tasks of varying angular alignment.
  }
  \label{fig:exp-data-geometry}
\end{figure}

\subsection{Optimization Bias} \label{sec:exp-optim-bias}
A second key factor for the success of distillation 
is a specific optimization bias.
For $n<d$, the distillation training objective~\eqref{eq:loss} 
has many minima of identical function value but potentially 
different generalization properties. Therefore, the 
optimization method used could have a large impact on the 
transfer risk.
As Theorems~\ref{thm:closed-form-N1} and~\ref{thm:closed-form}
show, gradient descent has a particularly favorable bias for 
distillation.

To verify this observation experimentally, we consider learners that 
are guided by an optimisation bias to different degrees:
at one end of the spectrum is the gradient-descent learner 
we have studied in previous sections, while at the other 
end is a learner that treats all minimizers of the distillation 
training loss equally, \ie that has no bias toward any of the 
solutions.
Specifically, consider learners with weights of the form
$\fw_\delta = \hat\fw + \delta\frac{\norm{\hat\fw}}{\norm\fq} \fq$,
where $\hat\fw$ is the gradient-descent distillation solution and $\fq$ 
is a Gaussian random vector in the subspace orthogonal to the data span, 
\ie if $\fX$ is the data matrix, then $\fX\T\fq = \f0$.
All learners of this form globally minimize the distillation training 
loss, and depending on $\delta$, they are more or less guided by the 
gradient-descent bias:
$\delta=0$ and $\abs{\delta} \to\infty$ represent the two extremes mentioned above.

We train the learners $\fw_\delta$ for $\delta\in\braces{0,10,\dots,90}$ 
on the digits $0$ and $1$ of the MNIST dataset, where inputs are 
treated as vectors in $\R^{784}$ and the teacher $\fw_*$ is a 
logistic regression trained to classify 0s and 1s on an 
independent training set.
We set the transfer set size to $n=100$ and evaluate the risk 
on the test set.

Figure~\ref{fig:exp-optim-bias} shows the result.
There is a clear trend in favor of learners that are more strongly 
guided by the gradient-descent bias (small $\delta$); these learners 
generally achieve lower transfer risk.
This result supports the idea of optimization bias as a key component 
of distillation's success.

\begin{figure}[t]
  \centering
  \includegraphics[scale=0.35]{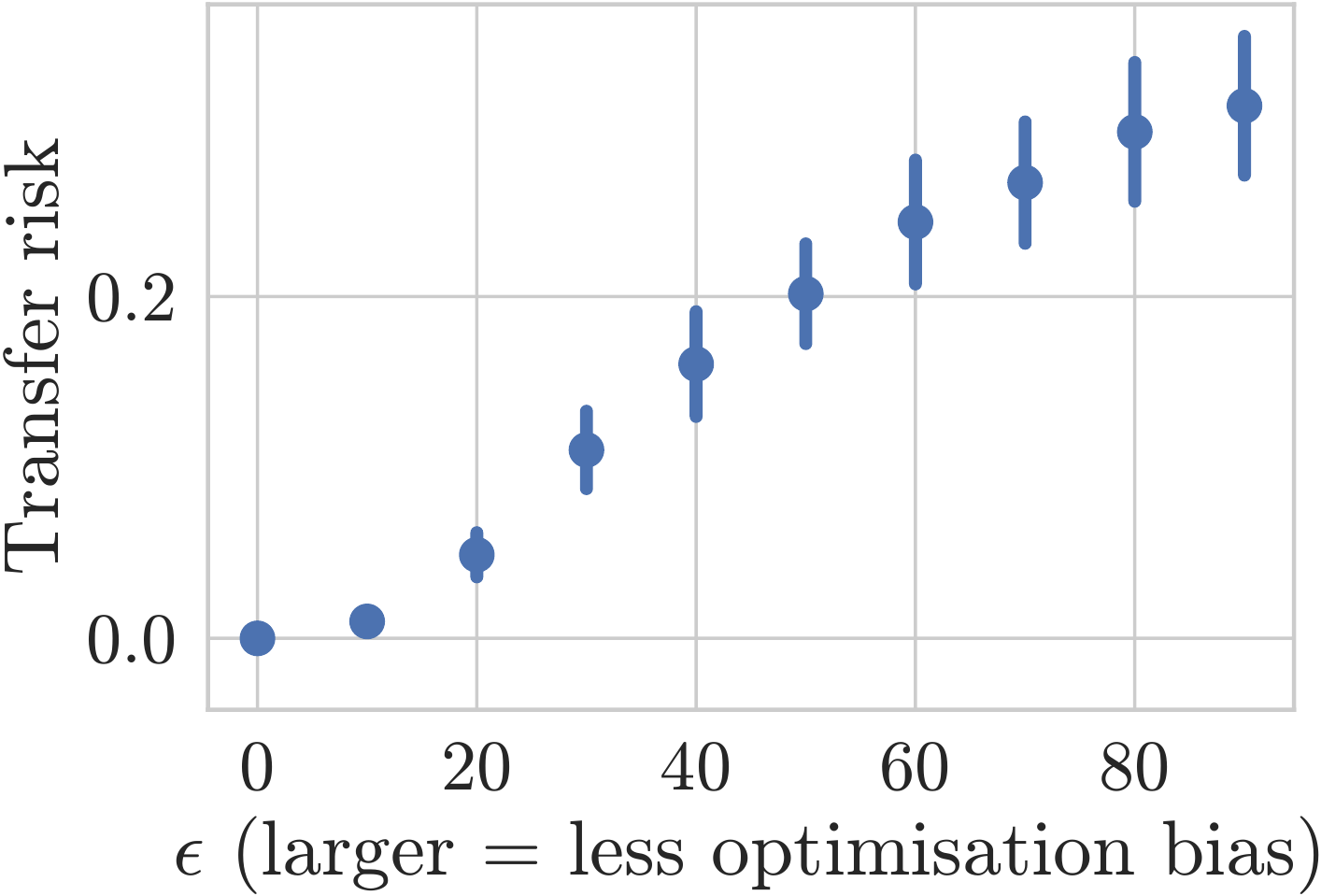}
  \caption{
    Transfer risk of linear distillation variants with different degrees of optimization bias, on the digits $0$ and $1$ of MNIST. 
  }
  \label{fig:exp-optim-bias}
\end{figure}

\subsection{Strong Monotonicity}
The third key factor we identify is \emph{strong monotonicity}, 
as established in Lemma~\ref{lemma:improvement}: training the student 
on more data always leads to a better approximation of the teacher's 
weight vector. 

Compared to data geometry and optimisation bias, strong monotonicity 
is less amenable to experimental study because it is a downstream 
property that cannot directly be manipulated. 
We therefore take an indirect approach.
We consider a set of learners including the gradient-descent distillation learner, the hard-target learner, and several learners with reduced optimisation bias (as in Section~\ref{sec:exp-optim-bias}), and train them on the same task.
For each learner, we note its expected risk calculated on 
a held-out set, and its \emph{monotonicity index}, defined as the 
probability that an additional training example  
reduces the angle between the student's and the teacher's 
weight vectors rather than increasing it, \ie 
%
\begin{equation}
  m(\fw) = \p_{\substack{ \fX\sim P_\fx^n \\ \fx\sim P_\fx}} \brackets{ \bar\alpha(\fw_*,\fw([\fX,\fx])) < \bar\alpha(\fw_*,\fw(\fX)) },
\end{equation}
where the student's weight vector $\fw$ is now treated as a function of the training set.
Thus, we can relate a learner's risk and its monotonicity.

We train the learners on the polynomial-angle task $(P_\fx^\kappa, \fw_*^\kappa)$ from Section~\ref{sec:exp-data-geom}, with $\kappa=1, d=100$ and $n=5$.
The expected risk as well as the monotonicity index are estimated as averages over 1000 transfer sets.

The results are shown in Figure~\ref{fig:exp-monotonicity}.
There is a negative correlation between monotonicity and transfer risk, 
which supports the intuition of monotonicity as a desirable property and a possible explanation of distillation's success.

However, a few reservations are in order. First, as mentioned above, 
monotonicity cannot easily be manipulated, so its \emph{effect} on 
transfer risk remains unknown. We can only measure correlation.
Second, monotonicity is of binary nature; it only captures \emph{whether} an extra data point helps or not.
Yet for a quantitative characterization of risk, one would have to capture \emph{by how much} an extra data point helps.
We leave more refined definitions of monotonicity for future work.

\begin{figure}[t]
  \centering
  \includegraphics[scale=0.45]{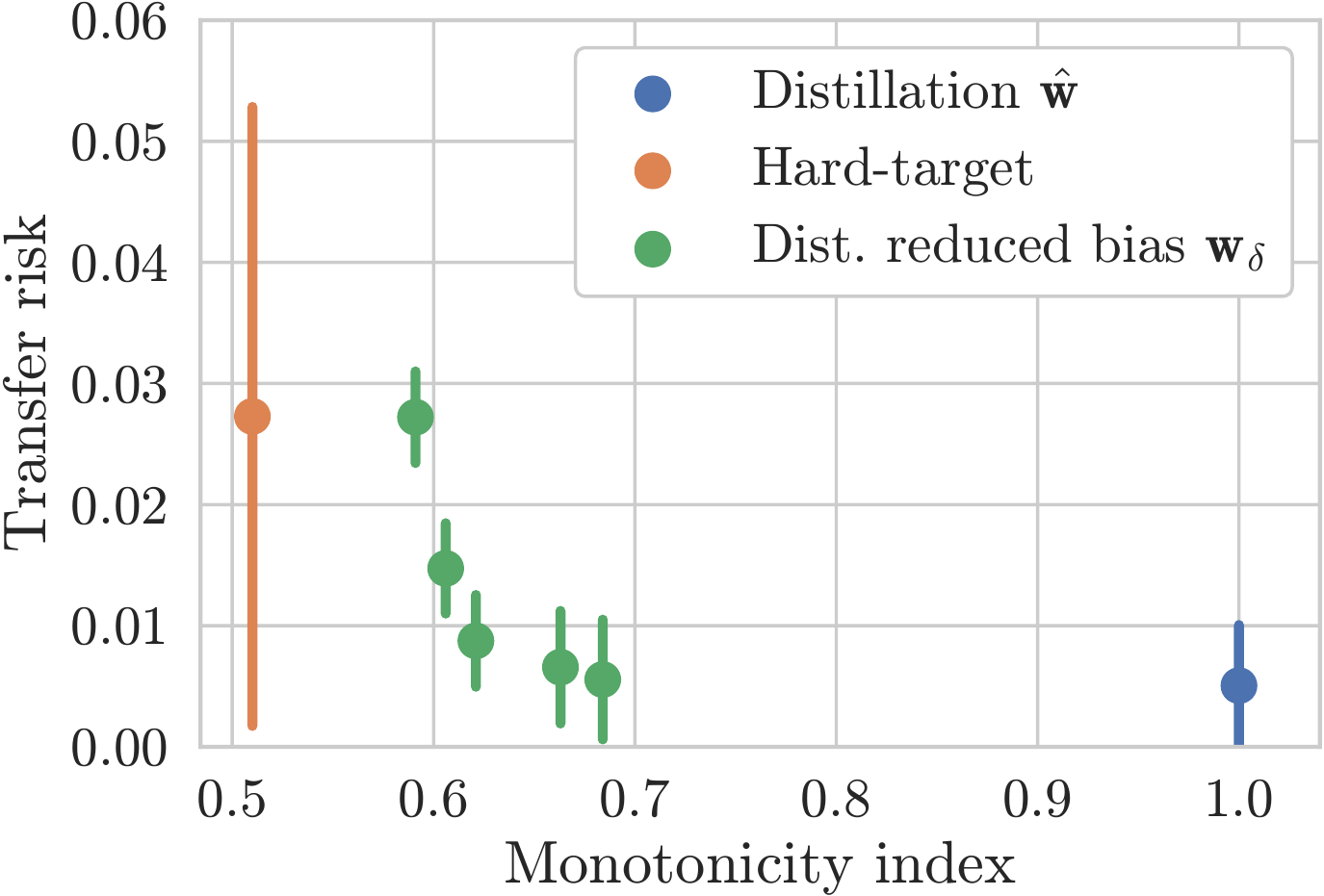}  
  \caption{
    Expected transfer risk vs. monotonicity of different learners: gradient-descent based distillation (blue), hard-target learner (orange), and a series of distillation learners with reduced optimisation bias (green): $\fw_\delta$ for $\delta\in\braces{1/16, 1/8, 1/4, 1/2, 1}$, listed in order from left to right.
  }
  \label{fig:exp-monotonicity}
\end{figure}

\section{Conclusion}

In this work, we have formulated and studied a linear model 
of knowledge distillation.
Within this model, we have derived
a)~\emph{a~characterization of the solution learned by the student},
b)~\emph{a~bound on the transfer risk, meaningful even in the 
low-data regime},
and c)~\emph{three key factors that explain the success of distillation}.
In doing so, we hope to have enriched both the current intuitive 
and theoretical understanding of distillation, both of which have 
only been weakly developed.
%

Our work paints a picture of distillation as an extremely effective 
method for knowledge transfer that derives its power from an 
optimization bias of gradient-based methods initialized near 
the origin, which in particular has the effect that any 
additionally included training point can only improve
the student's approximation of the teacher.
Distillation further benefits strongly from a favorable data 
geometry, in particular a margin between classes. 

While we have supported this picture by theory and empirical work 
only in the linear case, we hypothesize that similar properties 
also govern the behavior of distillation in the nonlinear setting.
If this hypothesis turns out to be true, it would have implications for the design of transfer sets (a large teacher model being stored along with only the minimal dataset necessary for future transfer) or active learning (which samples are most informative to have labeled by the teacher).
Potentially, strong monotonicity could serve as a leading design principle for new sample-efficient algorithms. 
We thus consider the extension to nonlinear models the main direction for future work.

\clearpage





\bibliography{b}

\begin{thebibliography}{29}
\providecommand{\natexlab}[1]{#1}
\providecommand{\url}[1]{\texttt{#1}}
\expandafter\ifx\csname urlstyle\endcsname\relax
  \providecommand{\doi}[1]{doi: #1}\else
  \providecommand{\doi}{doi: \begingroup \urlstyle{rm}\Url}\fi

\bibitem[Arora et~al.(2018)Arora, Cohen, and Hazan]{arora18}
Arora, S., Cohen, N., and Hazan, E.
\newblock On the optimization of deep networks: Implicit acceleration by
  overparameterization.
\newblock In \emph{International Conference on Machine Learing (ICML)}, 2018.

\bibitem[Arora et~al.(2019)Arora, Cohen, Golowich, and Hu]{arora19}
Arora, S., Cohen, N., Golowich, N., and Hu, W.
\newblock A convergence analysis of gradient descent for deep linear neural
  networks.
\newblock In \emph{International Conference on Learning Representations
  (ICLR)}, 2019.

\bibitem[Ba \& Caruana(2014)Ba and Caruana]{ba14}
Ba, J. and Caruana, R.
\newblock Do deep nets really need to be deep?
\newblock In \emph{Conference on Neural Information Processing Systems (NIPS)},
  2014.

\bibitem[Bucilua et~al.(2006)Bucilua, Caruana, and Niculescu-Mizil]{bucilua06}
Bucilua, C., Caruana, R., and Niculescu-Mizil, A.
\newblock Model compression.
\newblock In \emph{Conference on Knowledge Discovery and Data Mining (KDD)},
  2006.

\bibitem[Celik et~al.(2017)Celik, Lopez-Paz, and McDaniel]{celik2017patient}
Celik, Z.~B., Lopez-Paz, D., and McDaniel, P.
\newblock Patient-driven privacy control through generalized distillation.
\newblock In \emph{IEEE Symposium on Privacy-Aware Computing (PAC)}, 2017.

\bibitem[Craven \& Shavlik(1996)Craven and Shavlik]{craven96}
Craven, M. and Shavlik, J.~W.
\newblock Extracting tree-structured representations of trained networks.
\newblock In \emph{Conference on Neural Information Processing Systems (NIPS)},
  1996.

\bibitem[Geras et~al.(2016)Geras, Mohamed, Caruana, Urban, Wang, Aslan,
  Philipose, Richardson, and Sutton]{geras15}
Geras, K.~J., Mohamed, A.-R., Caruana, R., Urban, G., Wang, S., Aslan, O.,
  Philipose, M., Richardson, M., and Sutton, C.
\newblock Blending {LSTM}s into {CNN}s.
\newblock In \emph{International Conference on Learning Representations (ICLR)
  Workshop}, 2016.

\bibitem[Hardt \& Ma(2017)Hardt and Ma]{hardt17}
Hardt, M. and Ma, T.
\newblock Identity matters in deep learning.
\newblock In \emph{International Conference on Learning Representations
  (ICLR)}, 2017.

\bibitem[Hinton et~al.(2014)Hinton, Vinyals, and Dean]{hinton14}
Hinton, G., Vinyals, O., and Dean, J.
\newblock Distilling the knowledge in a neural network.
\newblock In \emph{Deep Learning Workshop at NIPS}, 2014.

\bibitem[Howard et~al.(2017)Howard, Zhu, Chen, Kalenichenko, Wang, Weyand,
  Andreetto, and Adam]{howard17}
Howard, A.~G., Zhu, M., Chen, B., Kalenichenko, D., Wang, W., Weyand, T.,
  Andreetto, M., and Adam, H.
\newblock Mobile{N}ets: Efficient convolutional neural networks for mobile
  vision applications.
\newblock In \emph{ar{X}iv:1704.04861}, 2017.

\bibitem[Hu et~al.(2016)Hu, Ma, Liu, Hovy, and Xing]{hu16}
Hu, Z., Ma, X., Liu, Z., Hovy, E., and Xing, E.
\newblock Harnessing deep neural networks with logic rules.
\newblock In \emph{Annual Meeting of the Association for Computational
  Linguistics (ACL)}, 2016.

\bibitem[Kawaguchi(2016)]{kawaguchi16}
Kawaguchi, K.
\newblock Deep learning without poor local minima.
\newblock In \emph{Conference on Neural Information Processing Systems (NIPS)},
  2016.

\bibitem[Li et~al.(2014)Li, Zhao, Huang, and Gong]{li14}
Li, J., Zhao, R., Huang, J.-T., and Gong, Y.
\newblock Learning small-size {DNN} with output-distribution-based criteria.
\newblock In \emph{Conference of the International Speech Communication
  Association (Interspeech)}, 2014.

\bibitem[Li et~al.(2017)Li, Yang, Song, Cao, Luo, and Li]{li17}
Li, Y., Yang, J., Song, Y., Cao, L., Luo, J., and Li, L.-J.
\newblock Learning from noisy labels with distillation.
\newblock In \emph{International Conference on Computer Vision (ICCV)}, 2017.

\bibitem[Liang et~al.(2008)Liang, Daum{\'e}~{III}, and Klein]{liang08}
Liang, P., Daum{\'e}~{III}, H., and Klein, D.
\newblock Structure compilation: trading structure for features.
\newblock In \emph{International Conference on Machine Learing (ICML)}, 2008.

\bibitem[Liu \& Zhu(2016)Liu and Zhu]{liu16}
Liu, J. and Zhu, X.
\newblock The teaching dimension of linear learners.
\newblock \emph{Journal of Machine Learning Research (JMLR)}, 17\penalty0
  (1):\penalty0 5631--5655, 2016.

\bibitem[Lopez-Paz et~al.(2016)Lopez-Paz, Bottou, Sch{\"o}lkopf, and
  Vapnik]{lopezpaz16}
Lopez-Paz, D., Bottou, L., Sch{\"o}lkopf, B., and Vapnik, V.
\newblock Unifying distillation and privileged information.
\newblock In \emph{International Conference on Learning Representations
  (ICLR)}, 2016.

\bibitem[Papernot et~al.(2016)Papernot, McDaniel, Wu, Jha, and
  Swami]{papernot16}
Papernot, N., McDaniel, P., Wu, X., Jha, S., and Swami, A.
\newblock Distillation as a defense to adversarial perturbations against deep
  neural networks.
\newblock In \emph{IEEE Symposium on Security and Privacy (S\&P)}, 2016.

\bibitem[Polino et~al.(2018)Polino, Pascanu, and Alistarh]{polino18}
Polino, A., Pascanu, R., and Alistarh, D.
\newblock Model compression via distillation and quantization.
\newblock In \emph{International Conference on Learning Representations
  (ICLR)}, 2018.

\bibitem[Romero et~al.(2015)Romero, Ballas, Kahou, Chassang, Gatta, and
  Bengio]{romero15}
Romero, A., Ballas, N., Kahou, S.~E., Chassang, A., Gatta, C., and Bengio, Y.
\newblock Fitnets: Hints for thin deep nets.
\newblock In \emph{International Conference on Learning Representations
  (ICLR)}, 2015.

\bibitem[Rusu et~al.(2016)Rusu, Colmenarejo, Gulcehre, Desjardins, Kirkpatrick,
  Pascanu, Mnih, Kavukcuoglu, and Hadsell]{rusu15}
Rusu, A.~A., Colmenarejo, S.~G., Gulcehre, C., Desjardins, G., Kirkpatrick, J.,
  Pascanu, R., Mnih, V., Kavukcuoglu, K., and Hadsell, R.
\newblock Policy distillation.
\newblock In \emph{International Conference on Learning Representations
  (ICLR)}, 2016.

\bibitem[Saxe et~al.(2014)Saxe, McClelland, and Ganguli]{saxe14}
Saxe, A.~M., McClelland, J.~L., and Ganguli, S.
\newblock Exact solutions to the nonlinear dynamics of learning in deep linear
  neural networks.
\newblock In \emph{International Conference on Learning Representations
  (ICLR)}, 2014.

\bibitem[Sch\"{o}lkopf \& Smola(2002)Sch\"{o}lkopf and
  Smola]{schoelkopf2002lwk}
Sch\"{o}lkopf, B. and Smola, A.~J.
\newblock \emph{Learning With Kernels}.
\newblock MIT Press, 2002.

\bibitem[Tang et~al.(2016)Tang, Wang, and Zhang]{tang16}
Tang, Z., Wang, D., and Zhang, Z.
\newblock Recurrent neural network training with dark knowledge transfer.
\newblock In \emph{IEEE International Conference on Acoustics, Speech and
  Signal Processing (ICASSP)}, 2016.

\bibitem[Urner et~al.(2011)Urner, Shalev-Shwartz, and Ben-David]{urner11}
Urner, R., Shalev-Shwartz, S., and Ben-David, S.
\newblock Access to unlabeled data can speed up prediction time.
\newblock In \emph{International Conference on Machine Learing (ICML)}, 2011.

\bibitem[Vapnik \& Izmailov(2015)Vapnik and Izmailov]{vapnik15}
Vapnik, V. and Izmailov, R.
\newblock Learning using privileged information: similarity control and
  knowledge transfer.
\newblock \emph{Journal of Machine Learning Research (JMLR)}, 16\penalty0
  (2):\penalty0 2023--2049, 2015.

\bibitem[Yu et~al.(2017)Yu, Li, Morariu, and Davis]{yu17}
Yu, R., Li, A., Morariu, V.~I., and Davis, L.~S.
\newblock Visual relationship detection with internal and external linguistic
  knowledge distillation.
\newblock In \emph{International Conference on Computer Vision (ICCV)}, 2017.

\bibitem[Zhu(2013)]{zhu13}
Zhu, J.
\newblock Machine teaching for {B}ayesian learners in the exponential family.
\newblock In \emph{Conference on Neural Information Processing Systems (NIPS)},
  2013.

\bibitem[Zhu(2015)]{zhu15}
Zhu, X.
\newblock Machine teaching: An inverse problem to machine learning and an
  approach toward optimal education.
\newblock In \emph{AAAI Conference on Artificial Intelligence}, 2015.

\end{thebibliography}
\bibliographystyle{icml2019}


\newpage
\appendix
\twocolumn[
\icmltitle{Supplementary Material}]

\numberwithin{assumption}{section}
\numberwithin{corollary}{section}
\numberwithin{lemma}{section}
\numberwithin{theorem}{section}

We define here some notation in addition to that of Section~\ref{sec:setup} in the main text.
We denote by $\ell_i$ the per-instance loss, 
\begin{align} \label{eqA:loss}
  L^1(\fw) &= \frac 1 n \sum_{i=1}^n \ell_i(\fw\T\fx_i), \\
  \ell_i(u) &= -y_i\log\sigma(u) - (1-y_i)\log(1-\sigma(u)) - \ell_i^*,
\end{align}
where $\ell_i^*$ are constants chosen such that the minimum of $\ell_i$ is 0, namely
$\ell_i^* = - y_i\log y_i - (1-y_i)\log(1-y_i)$.

Slightly abusing notation, we write
$L(\tau) = L^1(\fw(\tau)) = L(\fW_1(\tau),\dots,\fW_N(\tau)) $
for the objective value at time $\tau$.

Finally, for a full-rank matrix $\fA\in\R^{d\times m}\ (m\geq 1)$, we denote by $\fP_\fA\in \R^{d\times d}$ the matrix of projection onto the span of $\fA$,
\begin{equation}
  \fP_\fA = \left\{
  \begin{array}{cl}
    \fI, & m\geq d, \\
    \fA(\fA\T\fA)^{-1}\fA\T, & m < d.
  \end{array}\right.
\end{equation}

\section{Properties of the Cross-Entropy Loss}
\begin{theorem}[Gradient] \label{thm:gradient}
  The gradient of the cross-entropy loss~(\ref{eqA:loss}) takes the form
  \begin{equation}
    \nabla L^1(\fw) = \frac 1 n \sum_{i=1}^n (\sigma(\fw\T\fx_i)-y_i) \cdot \fx_i.
  \end{equation}
  It always lies in the data span, 
  $\nabla L^1(\fw) \in \mathrm{span}(\fX)$.  
\end{theorem}
\begin{proof}
  Straightforward calculation.
\end{proof}

\begin{theorem}[Global minima] \label{thm:global-opt}
  The global minimum of the cross-entropy loss~(\ref{eqA:loss}) is 0
  and the set of global minimisers is
  \begin{equation} \label{eq:set-of-opt}
    \braces{\fw\in\R^d: \fX\T\fw = \fX\T\fw_*}.
  \end{equation}
\end{theorem}
\begin{proof}
  We know that $L^1\geq 0$ and $L^1(\fw_*)=0$, so 0 is the optimal objective value, and the set of global optima consists of all $\fw$ such that $L^1(\fw)=0$.
  The last condition is equivalent to $\forall_i: \ell_i(\fw)=0$, which in turn is equivalent to $\forall_i: \sigma(\fw\T\fx_i) = \sigma(\fw_*\T\fx_i)$.
  By monotonicity of $\sigma$, this is further equivalent to $\forall_i:\fw\T\fx_i = \fw_*\T\fx_i$, which is a restatement of (\ref{eq:set-of-opt}).
\end{proof}

\begin{theorem}[Restricted strong convexity] \label{thm:strong-convexity}
  Assume $\fX$ is full-rank.
  For any sublevel set $\cW=\braces{\fw: L^1(\fw)\leq l}$, there exists $\mu>0$ such that 
  \begin{multline} \label{eq:A-strong-convexity}
    L^1(\fv) \geq L^1(\fw) + \nabla L^1(\fw)\T(\fv-\fw)
    + \frac \mu 2 \norm{\fv-\fw}^2
  \end{multline}
  for all $\fw,\fv\in\cW$ such that $\fv-\fw\in\mathrm{span}(\fX)$.
\end{theorem}
\begin{proof}
  Consider the 2nd-order Taylor expansion of $L^1$ around $\fw$,
  \begin{multline} \label{eq:taylor}
    L^1(\fv) = L^1(\fw) + \nabla L^1(\fw)\T(\fv - \fw) \\
    + \frac 1 2(\fv - \fw)\T [\nabla^2L^1(\bar\fw)] (\fv - \fw),
  \end{multline}
  where $\nabla^2L^1(\bar\fw)$ is the Hessian of $L^1$ evaluated at $\bar\fw$, a point lying between $\fv$ and $\fw$.
  A straightforward calculation shows that the Hessian takes the form
  \begin{equation}
   \nabla^2 L^1(\bar\fw) = \fX\fD_{\bar\fw}\fX\T, 
  \end{equation}
  where
  \begin{multline}
    \fD_{\bar\fw} = \mathrm{diag}[\sigma(\bar\fw\T\fx_1)(1-\sigma(\bar\fw\T\fx_1)), \\
    \dots,\sigma(\bar\fw\T\fx_n)(1-\sigma(\bar\fw\T\fx_n))].
  \end{multline}
  
  We will now show that there is a constant $\omega>0$ such that
  \begin{equation} \label{eq:var-bounded}
    \sigma(\bar\fw\T\fx_i)(1 - \sigma(\bar\fw\T\fx_i)) \geq \omega
  \end{equation}
  for all $\bar\fw\in\cW$ and $i\in\braces{1,\dots,n}$,
  so that we can claim $\fD_{\bar\fw}\succeq \omega\fI$,
  or consequently $\nabla^2 L^1(\bar\fw) \succeq \omega\fX\fX\T$.

  Let $\fw\in\cW$.
  The bound on $L^1(\fw)$ implies a bound on $\ell_i(\fw\T\fx_i)$ for all $i$,
  \begin{equation}
    \ell_i(\fw\T\fx_i) \leq nL^1(\fw) \leq nl.
  \end{equation}
  Because $\ell_i$ is convex and $\ell_i(u)\to\infty$ as $u\to\pm\infty$,
  we know that $\ell_i^{-1}((-\infty,nl])$ is a bounded interval, and the finite union
  $\cup_{i=1}^n \ell_i^{-1}((-\infty,nl])$ is also a bounded interval, whose size depends only on $nl$ and the data.
  Hence, there exists $K>0$ such that $\fw\T\fx_i \in [-K, K]$ for all $\fw\in\cW$ and $i\in\braces{1,\dots,n}$.
  The existence of $\omega>0$ satisfying (\ref{eq:var-bounded}) follows.

  Now, let us apply $\nabla^2L^1(\fw)\succeq \omega\fX\fX\T$ to lower-bound (\ref{eq:taylor}):
  \begin{multline} \label{eq:0}
    L^1(\fv) \geq L^1(\fw) + \nabla L^1(\fw)\T(\fv-\fw) \\
    + \frac \omega 2 (\fv - \fw)\T\fX\fX\T(\fv - \fw).
  \end{multline}
  Consider two cases.
  If $n\geq d$, $\fX\fX\T$ is full-rank and $\fX\fX\T \succeq \lambda_\mathrm{min}\fI$ holds, where $\lambda_\mathrm{min}>0$ is the smallest eigenvalue of $\fX\fX\T$.
  Combined with (\ref{eq:0}), this proves the claim for $n\geq d$ and $\mu=\omega\lambda_\mathrm{min}$.

  If $n<d$, $\fX\T\fX$ is full rank. We can use the assumption $\fv-\fw\in\mathrm{span}(\fX)$ to deduce
  \begin{align}
    \begin{split}
    \norm{\fv-\fw}^2 &= \norm{\fP_\fX(\fv-\fw)}^2 \\ 
                     &= (\fv-\fw)\T \fX(\fX\T\fX)^{-1}\fX\T (\fv-\fw) \\
                     &\leq \lambda_\mathrm{max} (\fv-\fw)\T\fX\fX\T(\fv-\fw),
    \end{split}
  \end{align}
  where $\lambda_\mathrm{max}>0$ is the largest eigenvalue of $(\fX\T\fX)^{-1}$.
  Combined with (\ref{eq:0}), this proves the claim for $n<d$ and $\mu=\omega/\lambda_\mathrm{max}$.
\end{proof}

\begin{corollary}[Restricted Polyak-Lojasiewicz] \label{cor:pl}
  Assume $\fX$ is full-rank.
  For any sublevel set $\cW = \braces{\fw: L^1(\fw)\leq l}$, there exists $c>0$ such that
  \begin{equation}
    cL^1(\fw) \leq \frac 1 2 \norm{\nabla L^1(\fw)}^2
  \end{equation}
  for all $\fw\in\cW$.
\end{corollary}
\begin{proof}
  Let $\fw\in\cW$. (If $\cW$ is empty, the claim is trivially true.)
  Theorem~\ref{thm:strong-convexity} applied to $\cW$ implies that for some $\mu>0$,
  \begin{equation} 
    L^1(\fv) \geq L^1(\fw) + \nabla L^1(\fw)\T(\fv-\fw) 
    + \frac {\mu} 2 \norm{\fv-\fw}^2
  \end{equation}
  for all $\fv\in\cW\cap\cV$ where $\cV = \braces{\fv: \fv-\fw\in\mathrm{span}(\fX)}$.
  Taking $\min_{\fv\in\cW\cap\cV}$ on both sides, then relaxing part of the constraint on the right-hand side yields
  \begin{align}
    \begin{split}
    &\min_{\fv\in\cW\cap\cV} L^1(\fv) \\
    &\hspace{5mm} \geq \min_{\fv\in\cW\cap\cV} L^1(\fw) + \nabla L^1(\fw)\T(\fv-\fw) + \frac \mu 2 \norm{\fv -\fw}^2 \\
    &\hspace{5mm} \geq \min_{\fv\in\cV} L^1(\fw) + \nabla L^1(\fw)\T(\fv-\fw) + \frac \mu 2 \norm{\fv -\fw}^2.
    \end{split}
  \end{align}

  Now, the minimum on the left-hand side is equal to 0 and is attained at
   $\fv = \fw + \fP_\fX(\fw_* - \fw)$,
 as can be seen from Theorem~\ref{thm:global-opt}.
 For the right-hand side, we can substitute $\fv = \fw + \fX\fa$ for $\fa\in\R^n$ and find the unconstrained minimum with respect to $\fa$.
 We get
 \begin{align}
   \begin{split}
   0 &\geq L^1(\fw) - \frac 1 {2\mu} \nabla L^1(\fw)\T \fX(\fX\T\fX)^{-1}\fX\T\nabla L^1(\fw) \\
   &\geq L^1(\fw) - \frac {\lambda_\mathrm{max}} {2\mu} \norm{\nabla L^1(\fw)}^2,     
   \end{split}
 \end{align}
 where $\lambda_\mathrm{max}>0$ is the largest eigenvalue of $\fX(\fX\T\fX)^{-1}\fX\T$.
 This yields the result with $c=\mu/\lambda_\mathrm{max}$.
\end{proof}

\section{Proof of Theorem~\ref{thm:closed-form-N1}}

We will prove a supporting lemma, and then the theorem.

\begin{lemma} \label{lemma:wt-in-span}
  Assume the student is a directly parameterised linear classifier $(N=1)$ initialised at zero, $\fw(0)=\f0$. Then, $\fw(\tau) \in\mathrm{span}(\fX)$ for $\tau\in[0,\infty)$.
\end{lemma}
\begin{proof}
  Let $\fq\in\R^d$ be any vector orthogonal to the span of $\fX$.
  It suffices to show that $\fq\T\fw(\tau) = 0$.
  For that, notice that $\fq\T\fw(0) = 0$ and
  \begin{equation}
    \dd{}{\tau} (\fq\T\fw(\tau)) = -\fq\T\nabla L^1(\fw(\tau)) = 0,
  \end{equation}
  where the last equality follows from the fact that $\nabla L^1(\fw(\tau))\in\mathrm{span}(\fX)$ (Theorem~\ref{thm:gradient}).
  The claim follows.
\end{proof}

\closedformshallow*
\begin{proof}
  Recall the time-derivative of $L$,
  \begin{equation} \label{eq:A-dLdt}
    L'(\tau) = -\norm{\nabla L^1(\fw(\tau))}^2.
  \end{equation}
  The data matrix $\fX$ is almost surely (wrt. $\fX\sim P_\fx^n$) full-rank,
  we can therefore apply Corollary~\ref{cor:pl} to $\cW = \braces{\fw: L^1(\fw)\leq L^1(\f0)}$ and $\fw(\tau)$ to lower-bound the gradient norm on the right-hand side of~(\ref{eq:A-dLdt}). We obtain  $L'(\tau) \leq -cL(\tau)$
  for some $c>0$ and all $\tau\in[0,\infty)$,
  or equivalently,
  \begin{align}
    (\log L(\tau))' \leq -c.
  \end{align}
  Integrating over $[0, t]$ yields $L(t) \leq L(0)\cdot e^{-ct}$, which proves global convergence in the objective: $L(t)\to 0$ as $t\to\infty$.

  Now invoke Theorem~\ref{thm:strong-convexity} with $\cW$ as above, $\fv=\fw(t)$ and $\fw=\hat\fw$ (we know that both $\fw(\tau),\hat\fw\in\cW\cap\mathrm{span}(\fX)$, partly by Lemma~\ref{lemma:wt-in-span}):
  \begin{equation} \label{eq:bound-dist}
    L(t) \geq \frac\mu 2 \norm{\fw(t) - \hat\fw}^2.
  \end{equation}
  Since $L(t)\to 0$ as $t\to\infty$, the theorem follows.
\end{proof}

\section{Proof of Theorem~\ref{thm:closed-form}}

\closedformdeep*

For the proof, we will need a result by~\cite{arora18}, which
characterises the induced flow on $\fw(\tau)$ when running gradient descent on the component matrices $\fW_i$.

\begin{lemma}[{\citep[Claim 2]{arora18}}] \label{lemma:w-deriv}
  If the balancedness condition~(\ref{ass:balancedness}) holds, then
  \begin{multline}
    \pd{\fw(\tau)}{\tau} = -\norm{\fw(\tau)}^{\frac{2(N-1)}{N}}
    \left(
      \nabla L^1(\fw(\tau)) + \right. \\ \left. (N-1)\cdot \fP_{\fw(\tau)} \nabla L^1(\fw(\tau)) 
    \right).
  \end{multline}
\end{lemma}


\begin{proof}[Proof of Theorem~\ref{thm:closed-form}]
  Similarly to the case $N=1$, we start by looking at the time-derivative of $L$,
  \begin{align}
    \begin{split} \label{eq:dLdt-N2}
    L'(\tau) =& \nabla L^1(\fw(\tau))\T\paren{\pd{\fw(\tau)}{\tau}} \\
    =& -\norm{\fw(\tau)}^{\frac{2(N-1)}{N}} \left( \norm{\nabla L^1(\fw(\tau))}^2 \right. \\
    & \left. + (N-1)\cdot \norm{ \fP_{\fw(\tau)}\nabla L^1(\fw(\tau))}^2 \right) \\
    \leq & -\norm{\fw(\tau)}^{\frac{2(N-1)}{N}} \cdot \norm{\nabla L^1(\fw(\tau))}^2 .
    \end{split}
  \end{align}
  It is non-positive, so $\fw(\tau)$ stays within the $L(0)$-sublevel set throughout optimisation,
  \begin{equation}
   \fw(\tau)\in\cW = \braces{\fw: L^1(\fw)\leq L(0)}.
 \end{equation}
 Also, $\cW$ is convex and by Assumption~(\ref{ass:better-than-zero}) it does not contain $\f0$.
 We can therefore take $\delta > 0$ to be the distance between $\cW$ and $\f0$,
 and it follows that $\norm{\fw(\tau)}\geq\delta$ for $\tau\in[0,\infty)$.
 
 Now, noting that $\fX$ is almost surely full-rank, apply Corollary~\ref{cor:pl} to $\cW$ and $\fw(\tau)$ to upper-bound the right-hand side of~(\ref{eq:dLdt-N2}),
 \begin{equation}
   L'(\tau) \leq - c\delta^{\frac{2(N-1)}{N}} L(\tau).
 \end{equation}
 Letting $\tilde c = c\delta^{\frac{2(N-1)}{N}}$, we get
 $(\log L(\tau))' \leq -\tilde c$
 and consequently $L(t)\leq L(0)\cdot e^{-\tilde c t}$.
 This proves convergence in the objective, $L(t)\to 0$ as $t\to\infty$.

 To prove convergence in parameters, we decompose the `error' $\fw(\tau) - \hat\fw$ into orthogonal components and bound each of them separately,
 \begin{multline} \label{eq:orth-decomposition}
   \norm{\fw(\tau) - \hat\fw}^2 = \norm{\fP_\fX(\fw(\tau)-\hat\fw)}^2 \\ + \norm{\fP_\fQ(\fw(\tau)-\hat\fw)}^2,
 \end{multline}
 where the columns of $\fQ\in\R^{d\times(d-n)}$ orthogonally complement those of $\fX$.
 If $n\geq d$, we simply bound the first term and disregard the second one.

 To bound the first term, invoke Theorem~\ref{thm:strong-convexity} with $\cW$, $\fv=\fP_\fX\fw(\tau)$ and $\fw=\fP_\fX\hat\fw$.
 One can check that $L^1(\fP_\fX\fu) = L^1(\fu)$ for all $\fu\in\R^d$, so $\fP_\fX\fw(\tau)\in\cW$ and our use of the theorem is legal.
 We obtain
 \begin{equation} \label{eq:norm-in-span}
   L(\tau) \geq \frac \mu 2 \norm{\fP_\fX(\fw(\tau) - \hat\fw)}^2.
 \end{equation}
 Since $L(\tau)\to 0$, it follows that
 \begin{equation} \label{eq:in-span-convergence}
  \norm{\fP_\fX(\fw(\tau) - \hat\fw)}^2\to 0
 \end{equation}
 as $\tau\to\infty$. 

 For the second term,
 notice that $\hat\fw\in\mathrm{span}(\fX)$, so $\fP_\fQ\hat\fw$ vanishes and we are left with $\norm{\fP_\fQ\fw(\tau)}^2$.
 Denote this quantity $q(\tau)$.
 Its time derivative is
 \begin{align}
   \begin{split}
     q'(\tau) =& \ 2(\fP_\fQ\fw(\tau))\T\paren{\pd{\fw(\tau)}{\tau}} \\
    =& -2\norm{\fw(\tau)}^{\frac{2(N-1)}N} \bigg( \fw(\tau)\T\fP_\fQ \nabla L^1(\fw(\tau)) + \\
   & \frac{(N-1)}{\norm{\fw(\tau)}^2}\cdot \fw(\tau)\T\fP_\fQ\fw(\tau) \cdot \fw(\tau)\T \nabla L^1(\fw(\tau)) \bigg) \\
   =& -2q(\tau) (N-1) \norm{\fw(\tau)}^{-2/N}  \fw(\tau)\T\nabla L^1(\fw(\tau)),
   \end{split}
 \end{align}
 where we have used the fact that $\nabla L^1(\fw(\tau))\in\mathrm{span}(\fX)$ (Theorem~\ref{thm:gradient}) and $\fQ$ is orthogonal to $\fX$.
 Rearranging, we obtain
 \begin{equation}
   \dd{}{\tau}\paren{\frac{\log q(\tau)}{2(N-1)}} = -\norm{\fw(\tau)}^{-2/N}\cdot \fw(\tau)\T \nabla L^1(\fw(\tau)).
 \end{equation}
 It turns out that the right-hand side expression is integrable in yet another way, namely
 \begin{multline}
   \dd{}{\tau}\paren{\frac 1 {2N} \log \norm{\fw(\tau)}^2 } = \\
   -\norm{\fw(\tau)}^{-2/N}\cdot \fw(\tau)\T \nabla L^1(\fw(\tau)).
 \end{multline}
 Equating the two and integrating over $[0,t]$ yields
 \begin{equation}
   \log\frac{q(t)}{q(0)} = \frac{N-1}N \cdot \log\frac{\norm{\fw(t)}^2}{\norm{\fw(0)}^2} ,
 \end{equation}
 which implies
 \begin{equation} \label{eq:qt}
   \frac{q(t)}{\norm{\fw(t)}^2} \leq \paren{\frac{\norm{\fw(0)}}{\norm{\fw(t)}}}^{2/N},
 \end{equation}
 because $q(0)\leq \norm{\fw(0)}^2$.

 We now bound the norm of $\fw(t)$.
 Starting from an orthogonal decomposition similar to (\ref{eq:orth-decomposition}) and applying (\ref{eq:in-span-convergence}) with (\ref{eq:qt}), we get
 \begin{align}
   \begin{split}
   \norm{\fw(t)}^2 = & \norm{\fP_\fX\fw(t)}^2 + \norm{\fP_\fQ\fw(t)}^2 \\
   \limsup_{t\to\infty}\norm{\fw(t)}^2 \leq & \norm{\hat\fw}^2 + \norm{\fw(0)}^{\frac 2 N} \limsup_{t\to\infty} \norm{\fw(t)}^{2-\frac{2}N}.
 \end{split}
 \end{align}
 Denote $\nu:=\limsup_{t\to\infty} \norm{\fw(t)}$. By the same orthogonal decomposition, we also know that $\nu^2 \geq \limsup_{t\to\infty} \norm{\fP_\fX\fw(t)}^2=\norm{\hat\fw}^2 > 0$, so we can divide both sides above by $\nu^2$,
 \begin{equation}
   1 \leq \frac{\norm{\hat\fw}^2}{\nu^2} + \frac{\norm{\fw(0)}^{2/N} }{\nu^{2/N}} =: f(\nu).
 \end{equation}
 On the right-hand side, we now have a decreasing function of $\nu$ that goes to zero as $\nu\to\infty$.
 However, evaluated at our specific $\nu$, it is lower-bounded by $1$, implying an implicit upper bound for $\nu$.

 How do we find this bound?
 Suppose we find some constant $K$ such that $f(K)\leq 1$.
 Then, because $f$ is decreasing, it must be the case that $\nu\leq K$.
 One such candidate for $K$ is
 \begin{equation}
   K = \norm{\hat\fw}\cdot \paren{ 1 - \frac{\norm{\fw(0)}^{2/N}}{\norm{\hat\fw}^{2/N}} } ^{\frac{-N}{2(N-1)}}.
 \end{equation}
 (Here we have used condition (\ref{ass:near-zero}): $\norm{\fw(0)} < \norm{\hat\fw}$.)
 To check that indeed $f(K)\leq 1$, start from the inequality
 \begin{multline}
   \paren{{\norm{\hat\fw}} /K}^{\frac{2(N-1)}N} + \frac{\norm{\fw(0)}^{2/N}}{\norm{\hat\fw}^{2/N}} = 1 \\
   \leq \paren{ 1 - \frac{\norm{\fw(0)}^{2/N}}{\norm{\hat\fw}^{2/N}} }^{\frac{-1}{N-1}}
   = (\norm{\hat\fw} / K)^{-\frac 2 N}.
 \end{multline}
 Taking the leftmost and rightmost expression and multiplying by $(\norm{\hat\fw}/K)^{2/N}$ yields
 \begin{equation}
   f(K) = \frac{\norm{\hat\fw}^2} {K^{2}} + \frac{\norm{\fw(0)}^{2/N}}{K^{2/N}} \leq 1.
 \end{equation}
 Hence,
 \begin{equation} \label{eq:norm-bound}
   \limsup_{t\to\infty} \norm{\fw(t)} \leq \norm{\hat\fw}\cdot \paren{ 1 - \frac{\norm{\fw(0)}^{2/N}}{\norm{\hat\fw}^{2/N}} } ^{\frac{-N}{2(N-1)}}.
 \end{equation}

 Finally, let us turn back to our original goal of bounding $\norm{\fw(\tau) - \hat\fw}^2$.
 With (\ref{eq:orth-decomposition}), (\ref{eq:in-span-convergence}), (\ref{eq:qt}) and (\ref{eq:norm-bound}), we now know that
 \begin{align}
   &\limsup_{t\to\infty} \norm{\fw(\tau) - \hat\fw}^2 \\
   & \leq \norm{\fw(0)}^{\frac 2 N}  \norm{\hat\fw}^{\frac{2(N-1)}{N}} \paren{ 1 - \frac{\norm{\fw(0)}^{\frac 2N}}{\norm{\hat\fw}^{\frac 2N}} } ^{-1} \\
   &= \frac{\norm{\hat\fw}^{2+2/N}}{\norm{\hat\fw}^{2/N} - \norm{\fw(0)}^{2/N} } - \norm{\hat\fw}^2.
 \end{align}
 Hence, if we initialise close enough to zero, as specified by condition~(\ref{ass:near-zero}), we can ensure that
 \begin{equation}
   \limsup_{t\to\infty} \norm{\fw(\tau)-\hat\fw}^2 < \epsilon^2.
 \end{equation}
 This concludes the proof.
\end{proof}

\section{Theorem~\ref{thm:bound} for Approximate Distillation}

We extend Theorem~\ref{thm:bound} to the setting where the student learns the solution
$\hat\fw = \fX(\fX\T\fX)^{-1}\fX\T\fw_*$
only $\epsilon$-approximately,
as is the case for deep linear networks initialised as in Theorem~\ref{thm:closed-form}.
When $n\geq d$, the teacher's weight vector is recovered exactly and the transfer risk is zero, even when the student is deep.
The following theorem therefore only covers the case $n<d$.

\begin{theorem}[Risk bound for approximate distillation] \label{thm:eps-bound}
Let $n<d$. For any training set $\fX\in\R^{d\times n}$, let $\hat h_\fX(\fx) = \ind\braces{\hat\fw_\epsilon\T\fx \geq 0}$ 
be a linear classifier whose weight vector is $\epsilon$-close to the distillation solution $\hat\fw$,
i.e. $\norm{\hat\fw_\epsilon - \hat\fw} \leq \epsilon$, 
where $\epsilon$ is a positive constant such that $\epsilon \leq \frac 1 2\norm{\hat\fw}$.
Define $\delta := \sqrt{\frac{2\pi\epsilon}{\norm{\hat\fw}}}$.
Then, it holds for any $\beta\in [0,\pi/2-\delta]$ that
\begin{equation}
\E_{\fX\sim P_\fx^{\otimes n}}\brackets{ R\big(\hat h_\fX \big| P_\fx,\fw_*\big)}
  \leq    p(\beta) + p(\pi/2-\delta-\beta)^n . 
\end{equation}
\end{theorem}

The result is very similar to Theorem~\ref{thm:bound} in the main text,
the only difference is the constant $\delta$ which compensates for the imprecision in learning $\hat\fw$ by pushing the bound up (recall that $p$ is decreasing).
However, as $\epsilon$ goes to zero, so does $\delta$ and we recover the original bound.

For the proof, we start with a tool for controlling the angle between $\hat\fw$ and $\hat\fw_\epsilon$.
Recall that the angle is defined as
\begin{equation}
  \alpha(\fw,\fv) = \cos^{-1}\paren{\frac{\fw\T\fv}{\norm{\fw}\cdot\norm{\fv}}}
\end{equation}
for $\fw,\fv\in\R^d \setminus \braces{\f0}$.

\begin{lemma} \label{lemma:small-angle}
  Let $\fw, \fv\in\R^d$ be such that $\norm{\fw - \fv} \leq \epsilon$, where $\epsilon \leq \frac 1 2\norm{\fw}$.
  Then $\alpha(\fw, \fv) \leq \sqrt{\frac{2\pi\epsilon}{\norm{\fw}}}$.
\end{lemma}
\begin{proof}[Proof of Lemma~\ref{lemma:small-angle}]
  The first step is to lower-bound the inner product $\fw\T\fv$.
  To that end, we expand and rearrange $\norm{\fw - \fv}^2 \leq \epsilon^2$ to obtain
  \begin{equation} \label{eq:2}
    2 \fw\T\fv \geq \norm{\fw}^2 + \norm{\fv}^2 - \epsilon^2.
  \end{equation}
  Now use the triangle relation $\norm{\fv} \geq \norm{\fw} - \epsilon$ squared to lower-bound the right-hand side of (\ref{eq:2}) and get
  \begin{equation}
    2 \fw\T\fv
    \geq 2 \norm{\fw}^2 -2\epsilon\norm{\fw},
  \end{equation}
  which implies
  \begin{equation}
    \frac{\fw\T\fv}{\norm{\fw}\cdot\norm{\fv}}
    \geq \frac{\norm{\fw}-\epsilon}{\norm{\fv}}
    \geq \frac{\norm{\fw}-\epsilon}{\norm{\fw}+\epsilon}
    \geq 1 - \frac{2\epsilon}{\norm{\fw}}.
  \end{equation}
  Thus,
  \begin{equation}
1 - \frac{2\epsilon}{\norm{\fw}} \leq \frac{\fw\T\fv}{\norm{\fw}\cdot\norm{\fv}} = \cos(\alpha(\fw,\fv)).
  \end{equation}
  The left-hand side is by assumption non-negative, so we have $\alpha(\fw,\fv)\in [-\pi/2, \pi/2]$.
  On this domain,
  \begin{equation}
   \cos x \leq 1 - \frac{x^2}{\pi},
 \end{equation}
 which lets us deduce
 \begin{equation}
   1 - \frac{2\epsilon}{\norm{\fw}} \leq 1 - \frac{\alpha(\fw,\fv)^2}{\pi}.
 \end{equation}
 Rearranging yields the result.
\end{proof}

\begin{proof}[Proof of Theorem~\ref{thm:eps-bound}]
  We decompose the expected risk as follows:
  \begin{align}\begin{split} \label{eqA:risk-decomposition}
      \E_{\fX\sim P_\fx^n}&\brackets{R\big(\hat h_\fX\big| P_\fx, \fw_* \big)} = 
      \p_{\substack{\fX\sim P_{\fx}^n \\ \fx\sim P_{\fx} }}[\fw_*\T \fx \cdot \hat\fw_\epsilon\T \fx < 0] = \\
      &= \int_{\fx: \bar\alpha(\fw_*, \fx) \geq \beta} \p_{\fX\sim P_{\fx}^n}[\fw_*\T \fx\cdot \hat\fw_\epsilon\T \fx <0 |\fx]\, \mathrm{d}P_{\fx} \\
      &+ \int_{\fx: \bar\alpha(\fw_*,\fx) < \beta ,\, \fw_*\T \fx > 0} \p_{\fX\sim P_{\fx}^n}[\hat\fw_\epsilon\T \fx <0 |\fx]\, \mathrm{d}P_{\fx} \\      
      &+ \int_{\fx: \bar\alpha(\fw_*,\fx) < \beta ,\, \fw_*\T \fx < 0} \p_{\fX\sim P_{\fx}^n}[\hat\fw_\epsilon\T \fx >0 |\fx]\, \mathrm{d}P_{\fx} .
    \end{split}\end{align}
  Let us fix some $\fx$ for which $\bar\alpha(\fw_*, \fx)<\beta$ and $\fw_*\T \fx > 0$; for this $\fx$ we have $\alpha(\fw_*,\fx)=\bar\alpha(\fw_*,\fx)$.
  Consider the situation where $\bar\alpha(\fw_*, \fx_i)< \pi/2-\beta -\delta$ for some $i$.
  Then by the triangle inequality, Lemma~\ref{lemma:small-angle} and Lemma~\ref{lemma:improvement},
  \begin{align}
    \alpha(\hat\fw_\epsilon,\fx) &\leq \alpha(\hat\fw_\epsilon,\hat\fw) + \alpha(\fw_*, \hat\fw) + \alpha(\fw_*,\fx) \\
                                 &\leq \delta + \bar\alpha(\fw_*,\fx_i) + \bar\alpha(\fw_*,\fx) \\
    &< \pi/2,
  \end{align}
  which implies $\hat\fw_\epsilon\T \fx > 0$, i.e. a correct prediction (same as the teacher's).
  Conversely, an error can occur only if $\bar\alpha(\fw_*,\fx_i) \geq \pi/2 - \delta - \beta $ for all $i$.
  Because $\fx_i$ are independent, we have
  \begin{align}
    \begin{split} \label{eqA:bound-trainset1}
    \p_{\fX\sim P_{\fx}^n}[\hat\fw_\epsilon\T \fx < &\, 0 |\,\fx: \bar\alpha(\fw_*,\fx)<\beta,\,  \fw_*\T \fx>0] \\
    &\leq \p_{\fX\sim P_{\fx}^n}[\forall_i: \bar\alpha(\fw_*,\fx_i) \geq \pi/2 - \delta - \beta] \\
    &= p(\pi/2-\delta-\beta)^n.
    \end{split}
  \end{align}
  By a symmetric argument, one can show that
  \begin{multline} \label{eqA:bound-trainset2}
    \p_{\fX\sim P_{\fx}^n}[\hat\fw_\epsilon\T \fx >0 |\,\fx: \bar\alpha(\fw_*,\fx)<\beta,\,  \fw_*\T \fx<0] \\ \leq p(\pi/2-\delta-\beta)^n.
  \end{multline}
  Combining (\ref{eqA:risk-decomposition}), (\ref{eqA:bound-trainset1}) and (\ref{eqA:bound-trainset2}) yields the result.
\end{proof}

\end{document}